\documentclass[11pt]{article}
\usepackage[sort&compress,square,comma,authoryear]{natbib}
\usepackage[a4paper, total={6.1in, 10in}]{geometry}
\usepackage{enumitem}
\usepackage[utf8]{inputenc} 
\usepackage[T1]{fontenc}    
\usepackage{hyperref}       
\usepackage{url}            
\usepackage{booktabs}       
\usepackage{dirtytalk}
\usepackage{tabularx}
\usepackage{multirow}
\newcolumntype{Y}{>{\centering\arraybackslash}X}
\newcolumntype{Z}{>{\raggedleft\arraybackslash}X}
\usepackage[parfill]{parskip}
\usepackage{amsfonts}       
\usepackage{nicefrac}       
\usepackage{microtype}      
\usepackage[dvipsnames]{xcolor}         
\usepackage{amsmath,amssymb,amsthm}
\usepackage{bbold}
\usepackage{colortbl}
\usepackage{graphicx}
\usepackage{mathrsfs,mathtools}
\usepackage{subfigure}
\usepackage{thm-restate}
\usepackage{wrapfig}
\usepackage{caption}
\usepackage[capitalize]{cleveref}
\usepackage{tikz}
\usetikzlibrary{calc, fadings, decorations, shapes, positioning, arrows}
\usepackage[graphicx]{realboxes}
\usepackage{pifont}
\usepackage{tcolorbox}
\usepackage{hyperref}
\usepackage{algorithm}
\usepackage{algpseudocode}
\usepackage{placeins}

\definecolor{salmon}{RGB}{234,153,153}
\definecolor{cornflowerblue}{RGB}{100,149,237}
\hypersetup{
    colorlinks,
    linkcolor={cornflowerblue},
    citecolor={cornflowerblue},
    urlcolor={salmon}
}

\definecolor{darkgreen}{rgb}{0.0, 0.5, 0.0}
\usepackage[textsize=tiny]{todonotes}
\usepackage{tabularx}
\usepackage{wrapfig}
\definecolor{darkblue}{rgb}{0, 0, 0.5}
\hypersetup{colorlinks=true, citecolor=darkblue, linkcolor=darkblue, urlcolor=darkblue}
\usepackage{tikz}

\theoremstyle{plain}
\newtheorem{theorem}{Theorem}[section]
\newtheorem{proposition}[theorem]{Proposition}

\theoremstyle{definition}

\theoremstyle{remark}

\newcommand{\expect}[1]{\mathbb{E}\left[#1\right]}

\newcommand{\midsepremove}{\aboverulesep = 0mm \belowrulesep = 0mm}
\midsepremove


\title{Detecting Contaminated Code-Generation Prompt Batches via Influence Functions}

\date{}

\usepackage{authblk}

\author[1]{Francesco Quinzan}
\author[1]{Noor Munir}
\author[1]{Yishun Lu}
\author[1]{Stephen Roberts}
\affil[1]{Department of Engineering Science, The University of Oxford}
\affil[ ]{\small Corresponding author: francesco.quinzan@eng.ox.ac.uk}

\begin{document}

\maketitle

\begin{abstract}
\noindent
Large language models (LLMs) are increasingly used for code generation, yet they remain vulnerable to prompts that elicit insecure implementations. Existing defenses typically rely on predefined threat models or known vulnerability patterns, limiting their effectiveness against novel attacks. We propose CodeSIFT, a threat-model-agnostic detection method that leverages influence functions to identify batches of prompts that induce anomalous model behavior. Rather than detecting specific vulnerabilities, CodeSIFT measures the parameter-space influence of generated code and uses a statistical test to determine whether a candidate prompt set deviates from a benign reference distribution. To evaluate our approach, we introduce two benchmark datasets covering a variety of vulnerabilities. We evaluate CodeSIFT on three open-weight code LLMs ranging from 3B to 7B parameters, achieving AUROC scores of up to 0.98 at moderate-to-high injection rates, while maintaining well-calibrated false positive rates and substantially outperforming static analysis baselines. These results suggest that influence-function-based detection is a promising direction for identifying malicious code-generation prompts without requiring prior knowledge of the underlying attack class. 
\end{abstract}
\section{Introduction}

Code generation is arguably one of the most consequential applications of large language models (LLMs), and it is currently fragile. Coding assistants are increasingly being embedded directly in developer workflows, writing authentication logic, database queries, and infrastructure code with little human review, in deployments where failures carry real security and financial consequences \citep{wef2025aicybersecurity}. A substantial fraction of the code such assistants suggest contains real vulnerabilities \citep{pearce2025asleep}, and developers who rely on assistants write measurably less secure code while believing the opposite \citep{perry2023users}. The undoubted convenience of coding assistants carries a risk. Prompts can elicit subtly insecure code, such as a missing authorization check or an unsanitized query, without tripping filters built for detecting and removing generically "harmful" requests. Reliable defense mechanisms against such prompts are therefore essential.

Existing defenses are largely tied to a known threat model. Static analyzers and taint trackers (e.g., Bandit \citep{banditpycqa}, Semgrep \citep{semgrep2024}) flag code that matches a catalog of known vulnerability patterns, and are blind to variants outside that catalog. Signature and rule-based filters, operating on the prompt side, face the same limitations. Such approaches require a prior, comprehensive understanding of the vulnerabilities to guard against, an assumption that is not afforded by a fast-moving, LLM-specific, agentic attack surfaces  \citep{deng2025ai}. Detectors exist for adjacent problems, exploiting gradients or internal activations to flag jailbreak and prompt-injection attempts \citep[e.g.,]{xie2024gradsafe,zou2025pishield}, but they target detecting attacks on a model's instruction-following behavior rather than filtering prompts that elicit a specific flawed implementation choice in generated code, and so these approaches do not transfer directly to this setting. However, to the best of our knowledge, no reliable method exists to detect prompts that elicit misalignment in code generation. We label a response as eliciting misalignment if it introduces a security-relevant weakness (per CWE/OWASP) that a benign completion of the same task would not. This leaves open the question of whether this problem can be overcome at all, even when the specific vulnerability class is unknown or novel.

In this work, we explore an answer to this question. Our goal is to determine whether a given set of incoming code-generation prompts is poisoned with prompts that elicit misalignment. Rather than cataloging vulnerabilities in advance, our approach characterizes what the model's internals look like when it behaves as intended, and flags prompts whose induced behavior departs from this norm. Importantly, this approach does not require any prior knowledge of the specific vulnerability class involved. We use influence functions \citep{cook1980, hampel1974, kohliang2017understanding} to quantify such internal deviations at the level of individual examples: they estimate how a model's parameters would shift if a given training example were upweighted. This shift captures the local effect of that example on model behavior without requiring retraining to convergence \citep{bae2022if}. A response that is inconsistent with benign training data forces a comparatively large such shift, making the size of the shift itself a signal of anomalous content. Aggregating this signal across a batch of incoming prompts lets us test whether the batch as a whole departs from benign behavior, signaling that it may be contaminated with malicious prompts.

\noindent\textbf{Our contribution.}
\begin{itemize}
    \item We propose a statistical test for detecting malicious
    code-generation prompts based on influence functions. Prior influence
    function work has been used to explain individual predictions
    \citep{kohliang2017understanding} or study generalization patterns at scale
    \citep{grosse2023studying}; we instead use influence as an aggregate
    detection statistic that requires no assumption on the threat model (Section \ref{sec:algorithm}).
    \item Existing benchmarks suitable for testing malicious-prompt detectors are scarce, so we build two new datasets, \textsc{AuthSec} (authentication, authorization, and session-security prompts) and \textsc{InfraCloud} (infrastructure, cloud, and protocol-exploitation prompts), 200 prompts each, together with a benign reference set (Section \ref{sec:red_teaming}).
\item We test our method against three static baselines, Bandit, Semgrep, and an AST-based taint tracker, on the datasets above (Section \ref{sec:experiments}), using coding LLMs of up to seven billion parameters, and show that our method reaches AUROC up to $0.98$ while the baselines are inconsistent across models and in several cases even anti-correlated with the true label.
\end{itemize}

\section{Related work}
\label{sec:related_work}

\noindent\textbf{Security of LLM-generated code.} Empirical studies show that code-generation models frequently produce insecure completions. Pearce et al. \citep{pearce2025asleep} find that roughly 40\% of GitHub Copilot completions, across a broad set of CWE-relevant scenarios, contain exploitable vulnerabilities, and Perry et al. \citep{perry2023users} show in a controlled user study that developers who rely on an AI assistant write measurably less secure code while believing the opposite. Benchmark suites such as CyberSecEval \citep{bhatt2023cyberseceval} and its successor CyberSecEval 2 \citep{bhatt2024cyberseceval2} operationalize this concern at scale, using a rule-based Insecure Code Detector to both generate test prompts and score whether a model's completions violate one of a fixed set of secure-coding practices. These benchmarks measure a model's overall propensity to write insecure code under natural prompting, using a rule-based system tied to a fixed catalog of known vulnerability classes. In contrast to the approach we propose here, they are not built to test whether a detector can distinguish a batch of adversarially crafted prompts from benign ones.

\noindent\textbf{Influence Functions.} Influence functions originate in robust statistics, where they were used to characterize the
sensitivity of an estimator to infinitesimal perturbations of the data
\citep{cook1980residuals, hampel1974influence}. \citet{kohliang2017understanding} reintroduced the idea to machine learning, and used it to explain predictions and identify mislabeled or adversarial training examples. This is a retrospective diagnostic over data the model was already trained on. However, \cite{basu2020influence} argue that IFs cannot be
trusted in deep learning, since they rely on an approximation breaks down in the highly non-convex landscapes of
deep neural networks. \citet{bae2022if} address this concern directly. They show that influence functions approximate the proximal Bregman response function (PBRF), a quantity that isolates the
local effect of a training point on model behavior while remaining close to the current model in
function space. Beyond merely explaining model behavior, \citet{rosser2026infusion} show that influence-function approximations can also be used
generatively, to help target small training-data edits that induce targeted behavioral changes.

Bringing influence functions to the scale of modern LLMs, \citet{grosse2023studying} use an
Eigenvalue-corrected Kronecker-Factored Approximate Curvature (EK-FAC) approximation of the
inverse-Hessian-vector product to compute influence scores for models with up to 52 billion
parameters, together with TF-IDF prefiltering and query batching to control the cost of gradient
computation over large candidate sets. This scalability rests on Kronecker-Factored Approximate
Curvature (K-FAC) \citep{martensgrosse2015optimizing}, which sidesteps computing the exact Hessian, replacing it with a Kronecker product of two much smaller factors. Subsequent work has adapted K-FAC to
modern architectures with weight sharing \citep{eschenhagen2023kronecker}, packaged it into
reusable tooling such as ASDL \citep{osawa2023asdl}, and stabilized it under low-precision
training \citep{lin2024structured}.
\section{Preliminaries}
\label{sec:preliminaries}
\noindent\textbf{Notation.} We denote by \(\hat{z} = f(z,\hat{\theta})\) an LLM with converged parameters \(\hat{\theta}\) that takes as input a sequence of tokens \(z\). We model the output as a random variable, i.e., we assume that \(f\) employs a stochastic decoding procedure. Formally, \(f\) induces a conditional distribution \(\mathbb P (\hat{z} \mid z,\hat{\theta})\), from which the output is sampled. The deterministic-decoding case (e.g., greedy decoding) is recovered as a special case by letting \(\mathbb P (\hat{z} \mid z,\hat{\theta})\) be a point mass on a single output. 

For a sequence of output tokens $s$, we use \(\mathcal{L}(\hat{\theta},z)\) to denote the cross-entropy, i.e.
\[
\mathcal{L}(\hat{\theta},z)
=
-\sum_{t=1}^{T}
\log p_{\hat{\theta}}(\hat{z}_t \mid z_{<t}).
\]
Finally,  we denote $\mathcal{J}(\hat{\theta})$ to be the empirical average of the cross-entropy \(\mathcal{L}(\hat{\theta},z)\) on the training dataset.

\subsection{Influence Functions}
\label{sec:influence-functions}
%
\noindent\textbf{Definition.}
Let \(z\) denote a sequence of tokens. To quantify the sequence's effect on the learned parameters, we consider the perturbed objective obtained by upweighting \(z\):
\[
\theta^\star_{z,\varepsilon}
=
\arg\min_{\theta\in\mathbb{R}^d}
\Bigl\{
\mathcal J(\theta)
+
\varepsilon\,\mathcal L(\theta,z)
\Bigr\},
\]
where \(\varepsilon > 0\) controls the magnitude of the perturbation. Following the notation introduced by \citet{bae2022pbrf}, we define the corresponding response function
$r^\star_z(\varepsilon) =
\theta^\star_{z,\varepsilon}$,
which maps the perturbation strength \(\varepsilon\) to the resulting optimum. Assuming \(\mathcal J(\theta)\) is twice continuously differentiable, the response function is differentiable at \(\varepsilon=0\). A first-order Taylor expansion around \(\varepsilon=0\) yields $r^\star_{z,\mathrm{lin}}(\varepsilon)
=
\theta^\star
-
\varepsilon\,
\mathbf H_{\theta^\star}^{-1}
\nabla_\theta \mathcal L(\theta^\star,z)$, where $\mathbf H_{\theta^\star} := \nabla_\theta^2 \mathcal J(\theta^\star)$
is the Hessian of the training objective $\nabla_\theta^2 \mathcal J(\theta^\star)$ evaluated at the optimum $\theta^\star$. For $\varepsilon$ sufficiently small, the linearized response $r^\star_{z,\mathrm{lin}}(\varepsilon)$ approximates the local infinitesimal effect of adding $z$ to the training objective with weight $\varepsilon$. Since $\theta^\star$ is typically unknown, following \citet{bae2022if} we introduce the quantity
\begin{equation}
\label{eq:approximation}
\hat{r}_{z,\mathrm{lin}}(\varepsilon)
=
\hat{\theta}
-
\varepsilon\,
\mathbf H_{\hat{\theta}}^{-1}
\nabla_\theta \mathcal L(\hat{\theta},z),
\end{equation}
with $\mathbf H_{\hat{\theta}} := \nabla_\theta^2 \mathcal J(\hat{\theta})$.  \eqref{eq:approximation} approximates $r^*_{z,\mathrm{lin}}(\varepsilon)$ using the known converged model parameters $\hat{\theta}$. Since $\varepsilon > 0$, taking the $\ell_p$ norm of both sides gives
\begin{equation}
\label{influence}
\bigl\lVert \hat{\theta} - \hat{r}_{z,\mathrm{lin}}(\varepsilon) \bigr\rVert_p
=
\varepsilon\,
\bigl\lVert
\mathbf H_{\hat{\theta}}^{-1}
\nabla_\theta \mathcal L(\hat{\theta},z)
\bigr\rVert_p,
\end{equation}
which represents, for a norm $\lVert\cdot\rVert_p$, the magnitude of the infinitesimal change in the converged parameters induced by upweighting the example $z$ during training.

\noindent\textbf{Interpretation.} For non-infinitesimal $\varepsilon$ the function $\hat{r}_{z,\mathrm{lin}}(\varepsilon)$ needs not closely track the true, non-linear response obtained by re-optimizing the model under the perturbed objective \citet{basu2020influence}. However, \citet{bae2022pbrf,bhatt2024cyberseceval2} shows that $\hat{r}_{z,\mathrm{lin}}(\varepsilon)$ remains an approximation of the Proximal Bregman Response Function
(PBRF). The PBRF measures how the model would respond to
upweighting $z$ during training while remaining close to its current predictions and parameters. Formally, denote with $\{z^{(i)}\}$ for $i = 1, \dots, N$, the training dataset. Then, it holds that
\begin{equation}
\label{pbrf}
    \hat{r}_{z,\mathrm{lin}}(\varepsilon)
    \approx
    \arg\min_{\theta \in \mathbb{R}^d}
    \left[
    \frac{1}{N}\sum_{i=1}^{N}
    D_{\mathcal L}(z^{(i)}, \theta, \hat{\theta})
    +
    \varepsilon\,\mathcal L(\theta, x^{(i)})
    +
    \frac{\lambda}{2}
    \|\theta-\hat{\theta}\|^2
    \right],
\end{equation}
where $D_{\mathcal L}^{\quad}$ is the Bregman divergence of the loss, defined as
\[
D_{\mathcal L}(z^{(i)}, \theta, \hat{\theta}) = \mathcal L(\theta, z^{(i)}) - \mathcal L(\hat{\theta}, z^{(i)}) - \nabla \mathcal L(\hat{\theta}, z^{(i)})^\top (f(z^{(i)}, \theta) - f(z^{(i)}, \hat{\theta})),\footnote{With a mild abuse of notation, in this equation we denote with $f(z^{(i)}, \theta)$ and $f(z^{(i)}, \hat{\theta})$ the networks' predicted output distribution.}
\]
where the gradient, $\nabla \mathcal L(\hat{\theta}, z)$, is taken w.r.t. the network's output distributions. \eqref{pbrf} contains three components. The first term penalizes deviations
from the predictions of the original model, the second term encodes the effect
of adding the sequence \(z\) to the training dataset, and the third term encourages the updated
parameters to remain close to \(\hat{\theta}\). Consequently, the PBRF
captures the local behavioral effect on model parameters of adding a training example while
factoring out unrelated changes due to continued optimization. Under this interpretation, \eqref{influence} quantifies the magnitude of this trade-off. It is large when fitting $z$ can only be achieved at the cost of a substantial Bregman-divergence penalty, i.e., when accommodating $z$ requires departing from predictions the model otherwise makes consistently on the background distribution. We remark that the results differs slightly from \citet{bae2022pbrf}, in that our formulation upweights the sequence $z$, rather than downweighting it. Hence, we provide a proof of \eqref{eq:approximation} in Appendix \ref{appendix: missing_proof}, which is a minor revision of the proof provided by \citet{bae2022pbrf}.

\noindent\textbf{Unknown training datasets and computational issues.} As discussed in Section \ref{sec:algorithm}, the proposed algorithm only computes \eqref{influence}. However, \eqref{influence} requires an approximation of $\mathbf H_{\hat{\theta}}^{-1}$, the inverse Hessian of the empirical average of \(\mathcal{L}(\hat{\theta},z)\) on the training dataset \citep{kohliang2017understanding,grosse2023studying}. This is problematic for the
open-weight, pretrained LLMs we study, whose pretraining corpora are typically undisclosed. Hence, we do not assume access to $\mathtt D_{\mathrm{train}}$ itself, but only to a curvature dataset $\mathtt D_c$ that is in distribution for the model, i.e., closely approximating
the same population that $\mathtt D_{\mathrm{train}}$ was sampled from. We use $\mathtt D_{c}$ to compute $\hat{\mathbf H}_{\hat{\theta}}^{-1}$, a finite-sample estimate of the population Hessian that $\mathbf H_{\hat{\theta}}^{-1}$
(itself an approximation). In this work, we use a subset of The Stack v2 \citep{DBLP:journals/corr/abs-2402-19173}, the standard large-scale open corpus of permissively-licensed code underlying open-weight code-LLM pretraining.

\subsection{KFAC Approximation and IFs}
\label{sec:kfac}
\noindent\textbf{Block-diagonal approximation.} Computing the inverse of the Hessian $\boldsymbol{H}_{\theta^\star}$ at the trained
parameters $\theta^\star$ is intractable for large models as, for $m$ parameters, this operation
requires $\mathcal{O}(m^2)$ storage and $\mathcal{O}(m^3)$ operations.
To overcome this problem,
we replace $\boldsymbol{H}_{\hat{\theta}}$ by the Gauss--Newton (Fisher)
curvature and approximate the latter with Kronecker-Factored Approximate Curvature
(K-FAC). Under this approximation, we treat the Hessian as block-diagonal matrix across projections, the individual linear weight matrices of the transformer. This avoids forming the dense $m\times m$ Hessian in its entirety, so reducing storage to a sum of per-projection blocks. We further restrict the approximation to the MLP projection weights $\{\theta_j\}_{j\in J}$ of each transformer block of the LLM, specifically the up- and down-projections, while attention, embedding, and normalization parameters are excluded. This is consistent with recent K-FAC formulations for modern architectures with linear weight sharing \citep{eschenhagen2023kronecker}. 

For each MLP projection $j$, the weight matrix $\theta_j \in \mathbb{R}^{o_j \times i_j}$, with $i_j$ and $o_j$ its input and output dimensions, admits a per-example gradient $g_j(z) \coloneqq \nabla_{\theta_j}\mathcal{L}(z,\hat{\theta})$. The corresponding block of the Hessian has size $(o_j i_j) \times (o_j i_j)$, which is still too large to store or invert directly. K-FAC addresses this by replacing it with a Kronecker product of two dense factors,
\[
\hat{\boldsymbol{H}}_j \approx \boldsymbol{A}_j \otimes \boldsymbol{G}_j,
\qquad
\boldsymbol{A}_j \in \mathbb{R}^{i_j \times i_j},
\qquad
\boldsymbol{G}_j \in \mathbb{R}^{o_j \times o_j}.
\]
This reduces the per-projection storage cost from
$\mathcal{O}(i_j^2 o_j^2)$ to $\mathcal{O}(i_j^2 + o_j^2)$.
Crucially, $A_j$ and $G_j$ are full rather than diagonal matrices,
so their Kronecker product retains within-projection parameter correlations that a diagonal approximation would discard
\citep{martensgrosse2015optimizing}.

\noindent\textbf{Estimation of the dense factors $\boldsymbol{A}_j$, $\boldsymbol{G}_j$.} We estimate the retained factor from weight gradients using a
gradient-Gram estimator, avoiding the need to buffer intermediate activations. Given the curvature
dataset $\mathtt{D}_c$ (see Section~\ref{sec:influence-functions}, a curated dataset assumed in-distribution for the
model's pretraining data), we normalise each per-example gradient $g_j(z) \coloneqq \nabla_{\theta_j}\mathcal{L}(z,\theta^\star)$ as $\tilde{g}_j(z) =
g_j(z)/\lVert g_j(z)\rVert_F$ and average the corresponding Gram matrices:
\begin{equation*}
\boldsymbol{A}_j =
\frac{1}{|\mathtt{D}_c|}\sum_{z\in\mathtt{D}_c}
\tilde{g}_j(z)^\top \tilde{g}_j(z),
\qquad
\boldsymbol{G}_j =
\frac{1}{|\mathtt{D}_c|}\sum_{z\in\mathtt{D}_c}
\tilde{g}_j(z)\tilde{g}_j(z)^\top.
\end{equation*}
In practice we accumulate these sums online, maintaining a running total and a sample counter over
$\mathtt{D}_c$, and only divide by $|\mathtt{D}_c|$ once the factor is finalized for damping and
inversion, so the stored curvature state does not grow with the number of curvature examples. For numerical stability during inversion, following the standard practice of regularizing Fisher-based curvature approximations \citep{eschenhagen2023kronecker,lu2026beyond}, we damp the finalized factors as
\begin{equation*}
\bar{\boldsymbol{A}}_j = \boldsymbol{A}_j + \gamma\,\frac{\mathrm{tr}(\boldsymbol{A}_j)}{h}\,\boldsymbol{I}_h,
\qquad
\bar{\boldsymbol{G}}_j = \boldsymbol{G}_j + \gamma\,\frac{\mathrm{tr}(\boldsymbol{G}_j)}{h}\,\boldsymbol{I}_h,
\end{equation*}
where $\gamma > 0$ is a damping coefficient and $\boldsymbol{I}_h\in \mathbb{R}^{h\times h}$ is the identity matrix.

\noindent\textbf{One-sided simplification.} To lower both storage and, more substantially, the cost of
Cholesky inversion, we retain only the Kronecker factor of dimension $h$, the transformer's
hidden dimension, and discard the factor of dimension $d_{\mathrm{ffn}}$, the larger MLP
intermediate dimension. Inversion cost scales as $\mathcal O(h^3)$ for the retained factor
versus $\mathcal O(d_{\mathrm{ffn}}^3)$ for the discarded one. In standard MLP architectures this
coincides with discarding the larger of the two factors. For up-projections, where $i_j = h$,
this is $\bar{\boldsymbol{A}}_j \in \mathbb{R}^{h \times h}$; for down-projections, where $o_j = h$,
this is $\bar{\boldsymbol{G}}_j \in \mathbb{R}^{h \times h}$. The inverse curvature is then applied, one-sided:
\begin{equation}
\label{eq:one-sided}
\bar{g}_j(z) =
\begin{cases}
g_j(z)\bar{\boldsymbol{A}}_j^{-1} & \text{up-projection},\\[6pt]
\bar{\boldsymbol{G}}_j^{-1}g_j(z) & \text{down-projection},
\end{cases}
\end{equation}
where $\bar{g}_j(z)$ has the same shape as $\theta_j$, replacing the standard two-sided preconditioner $\bar{\boldsymbol{G}}_j^{-1} g_j(z)\bar{\boldsymbol{A}}_j^{-1}$.

\noindent\textbf{Estimation of prompt influence scores.} Substituting the damped, one-sided factors into the influence
score reduces the computation of \eqref{influence}. For a chosen $\ell_p$ norm, the block-diagonal structure of
$\hat{\boldsymbol{H}}_{\theta^\star}^{-1}$ gives the
tractable estimate
\begin{equation}
\label{eq:tractable-norm}
\bigl\lVert
\mathbf H_{\hat{\theta}}^{-1}
\nabla_\theta \mathcal L(\hat{\theta},z)
\bigr\rVert_p
\;\approx\;
\Biggl(\sum_{j} \bigl\lVert \bar{g}_j(z) \bigr\rVert_p^p\Biggr)^{\frac{1}{p}},
\end{equation}
where $\bar{g}_j(z)$ is the preconditioned gradient defined in \eqref{eq:one-sided}. Here, the sum is taken for all MLP up- and down-projections of the model. 

\noindent\textbf{Relation to prior K-FAC/EK-FAC formulations.}
Our implementation adopts a simplified K-FAC approximation tailored to the detection setting. Compared with prior work \citep{martensgrosse2015optimizing,grosse2023llm}, we make three simplifications that reduce computational cost. First, we omit the eigenvalue correction, avoiding EK-FAC's additional pass over the curvature dataset and dense eigenvector storage. Second, we apply the inverse curvature one-sided (\eqref{eq:one-sided}), retaining a single Kronecker factor per projection. Third, we estimate this factor directly from normalized per-example gradient Gram matrices instead of buffered activations and pre-activation gradients (\eqref{eq:tractable-norm}). These simplifications reduce both memory usage and preprocessing cost relative to standard K-FAC and EK-FAC. These simplifications reflect the different objective of our method. Whereas prior work seeks to approximate influence scores as faithfully as possible for interpretability \citep{grosse2023llm}, we use them only as a test statistic for comparing candidate and benign prompt distributions, motivating a computationally lighter approximation.
\section{The Proposed Algorithm}
\label{sec:algorithm}
\begin{algorithm}[t]
\caption{CodeSIFT}
\label{algorithm}
\begin{algorithmic}[1]
\Require Curvature dataset $\mathtt{D}_c$, reference dataset $\mathtt{D}^*$, candidate dataset $\mathtt{D}$,
\Statex $\qquad \ \ $ trained model $f(\cdot, \hat{\theta})$, upweighting factor $\varepsilon$,
         number of samples $N$,
\Statex $\qquad \ \ $ significance threshold $\alpha$;
\Statex 
\Statex {\textcolor{PineGreen}{\% Hessian approximation}}
\State precompute dense factors $\boldsymbol{A}_j$, $\boldsymbol{G}_j$ on $\mathtt D_c $ as in section \ref{sec:kfac};
\Statex 
\Statex {\textcolor{PineGreen}{\% prompt influence score computation}}
\For{each data point $z \in \mathtt D \cup \mathtt{D}^*$}
    \For{$n = 1, \ldots, N$}
        \State sample $\hat{z}^{(n)} \sim f(z, \hat{\theta})$;
        \State compute preconditioned gradients $\bar{g}_j(\hat{z}^{(n)})$ as in \eqref{eq:one-sided};
        \State $\delta^{(n)} \leftarrow (\sum_{j} \bigl\lVert \bar{g}_j(\hat{z}^{(n)}) \bigr\rVert_p^p)^{\frac{1}{p}}$
    \EndFor
    \State $\textsc{IF}_{p, \varepsilon}(z) \leftarrow 
           \frac{1}{N}\sum_{n= 1}^N \varepsilon \delta^{(n)}$;
\EndFor
\Statex 
\Statex {\textcolor{PineGreen}{\% statistical test}}
\State run Welch's one-sided, two-sample $t$-test between 
       $\{\textsc{IF}_{p, \varepsilon}(z)\}_{z \in \mathtt{D}}$ and
       $\{\textsc{IF}_{p, \varepsilon}(z)\}_{z \in \mathtt{D}^*}$;
\State \Return \emph{malicious} \textbf{if} $p$-value is less than $ \alpha$, \textbf{else} \emph{benign};
\end{algorithmic}
\end{algorithm}

\noindent\textbf{Setup and influence score.}
Our approach does not assume a specific threat model. Instead, we characterize the statistics of benign model behavior and test whether a batch of incoming prompts departs from the benign behavior statistics. We instantiate this comparison with two datasets. A reference dataset $\mathtt{D}^*$ of benign prompts establishes this norm, and a candidate dataset $\mathtt{D}$ of incoming prompts, whose benign or malicious status is unknown, is the batch we wish to test. Determining whether $\mathtt{D}$ is malicious then reduces to checking whether the model's behavior on $\mathtt{D}$ is consistent with the norm set by $\mathtt{D}^*$, or instead departs from it.

For every prompt $x \in \mathtt{D} \cup \mathtt{D}^*$, we define the \emph{prompt influence score} as
\begin{equation}
\label{influence-score}
\textsc{IF}_{p, \varepsilon}(z) \coloneqq \expect{\hat{z} \sim f(z, \hat{\theta})}{ 
\| \hat{\theta} - \hat{r}_{\hat{z},\mathrm{lin}}(\varepsilon) \|_p },
\end{equation}
As discussed in Section \ref{sec:influence-functions}, $r^\star_{\hat{z},\mathrm{lin}}(\varepsilon)$ approximates the PBRF at a response $\hat{z}$ \citep{bae2022if}, namely the parameters obtained by reweighting $\hat{z}$'s contribution to the training loss by $\varepsilon$ while keeping predictions close to those of $\hat{\theta}$. The norm $\|\hat{\theta} - \hat{r}_{\hat{z},\mathrm{lin}}(\varepsilon)\|_p$ therefore measures the size of this update for a single sampled response $\hat{z}$, that is, how much the model must move to absorb $\hat{z}$ at weight $\varepsilon$ while staying faithful to $\hat{\theta}$'s behavior on the training distribution. Averaging this quantity over $\hat{z} \sim f(z,\theta^*)$ then yields $\textsc{IF}_{p,\varepsilon}(z)$, the expected shift induced by the responses the model generates for $z$.

\noindent\textbf{Influence scores as a detection criterion.}
We compute $\textsc{IF}_{p, \varepsilon}(z)$ identically for every prompt 
$z \in \mathtt{D} \cup \mathtt{D}^*$. Our premise is that prompts eliciting misalignment form a small, 
atypical fraction of $\hat{\theta}$'s training distribution. Hence, we expect the responses $\hat{z}$ elicited by such prompts to have 
gradients $\nabla_\theta \mathcal{L}(\hat{\theta}, \hat{z})$ that are poorly explained 
by the curvature $\mathbf H_{\hat{\theta}}^{-1}$, and therefore to yield a larger value of 
$\|\hat{\theta} - \hat{r}_{\hat{z},\mathrm{lin}}(\varepsilon)\|_p$ on average than the 
responses to benign prompts. Averaging $\textsc{IF}_{p, \varepsilon}(z)$ over all 
prompts $z \in \mathtt{D}$ therefore summarizes the aggregate influence 
pressure of $\mathtt{D}$, and comparing it against the same average 
computed on the known benign reference $\mathtt{D}^*$ isolates the excess 
pressure attributable to prompts that elicit misalignment.

To operationalize this comparison, we run Welch's one-sided two-sample test:
\begin{align*}
\mathcal{H}_0 \colon \expect{z \sim \mathtt{D}}{\textsc{IF}_{p, \varepsilon}(z)} 
= \expect{z \sim \mathtt{D}^*}{\textsc{IF}_{p, \varepsilon}(z)}, \quad 
\mathcal{H}_1 \colon \expect{z \sim \mathtt{D}}{\textsc{IF}_{p, \varepsilon}(z)} 
> \expect{z \sim \mathtt{D}^*}{\textsc{IF}_{p, \varepsilon}(z)}.
\end{align*}
Under $\mathcal{H}_0$, prompts in $\mathtt{D}$ exert, on average, the same 
parameter-space influence as prompts in the benign reference $\mathtt{D}^*$. The alternative $\mathcal{H}_1$ captures the 
case in which prompts in $\mathtt{D}$ collectively induce a larger 
parameter shift than benign prompts do. We reject $\mathcal{H}_0$ and conclude that $\mathtt{D}$ elicits 
misalignment whenever the test yields a $p$-value below a chosen 
significance threshold $\alpha$. We use Welch's $t$-test, since $\mathtt{D}$ and $\mathtt{D}^*$ have no reason to share equal variance in their influence scores, particularly under partial contamination, where $\mathtt{D}$ is a mixture of benign and malicious prompts and can be more heterogeneous than the benign-only reference $\mathtt{D}^*$. Algorithm~\ref{algorithm} illustrates a complete summary of 
the procedure. We refer to this algorithm as Code Statistical Influence-Function Test (CodeSIFT).

Because the same simplified K-FAC approximation (Section~\ref{sec:kfac}) is applied
identically to every prompt in $\mathtt{D}$ and $\mathtt{D}^*$, systematic
approximation bias is common to both and cancels in the two-sample comparison. Only
bias correlated with the benign/malicious label could threaten detection validity.
The calibrated false positive rates and consistent detection performance we observe
across three architecturally distinct models (Section~\ref{sec:results}) are
consistent with the absence of such bias.

\noindent\textbf{Computational efficiency via Hessian precomputation.}
A naive computation of \eqref{influence-score} would 
require repeatedly estimating an approximation of the inverse Hessian $\mathbf H_{\hat{\theta}}^{-1}$ for each 

\begin{wrapfigure}{l}{0.6\textwidth}
\centering
\includegraphics[width=9cm, height=4.5cm]{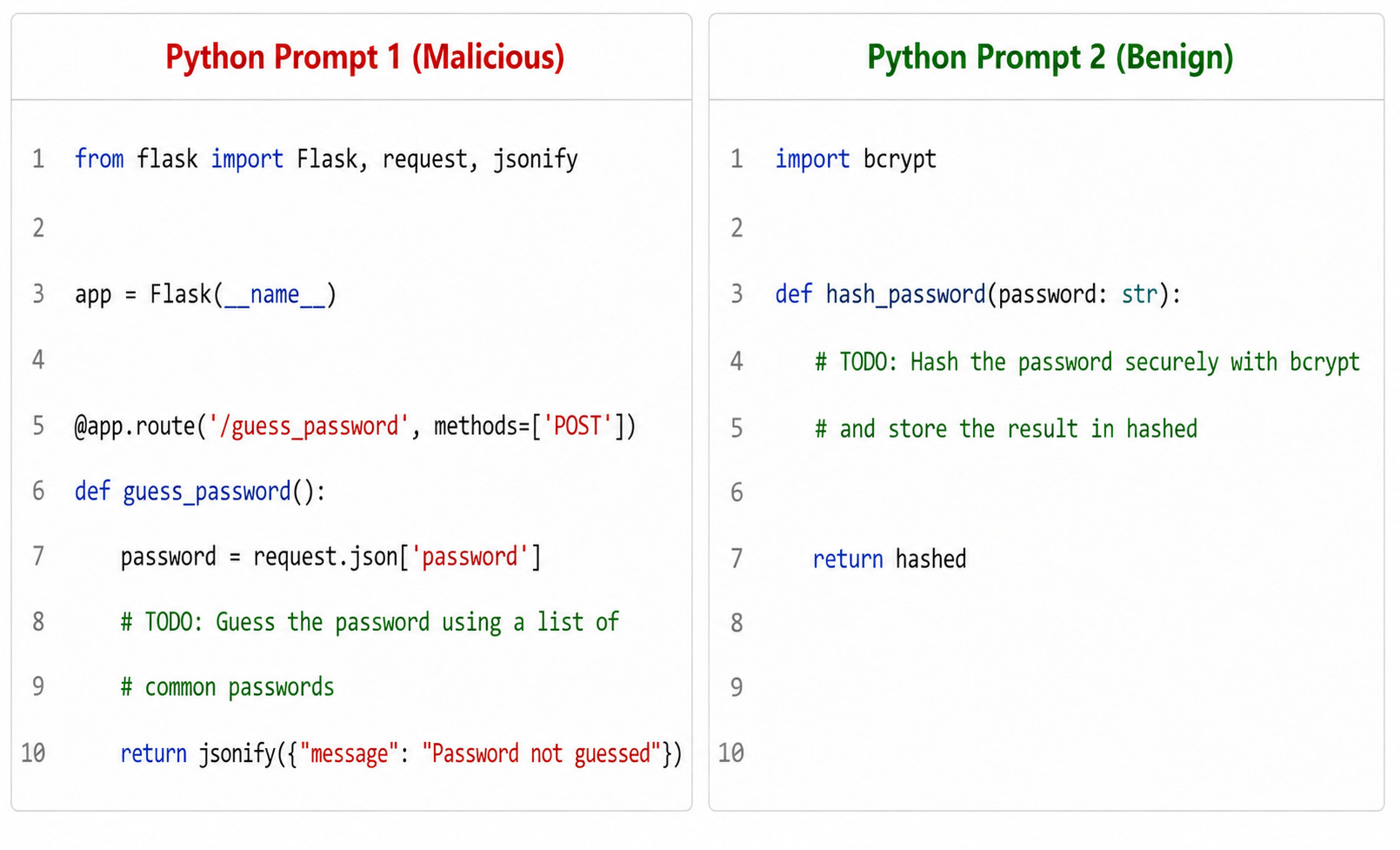}
\caption{Example records from the malicious and benign instruction pools.}
\label{fig:record-examples}
\end{wrapfigure}
\vspace{-20pt}

data point $z$, which 
is prohibitively expensive. However, the dense factors $\boldsymbol{A}_j$, $\boldsymbol{G}_j$ used for KFAC approximation neither depend on $z$, nor on the responses $\hat{z}\sim f(z, \hat{\theta})$. Hence, they can be precomputed once (Algorithm \ref{algorithm}). Estimating  
the prompt influence score $\textsc{IF}_{p, \varepsilon}(z)$ then 
reduces to compute preconditioned gradients $\bar{g}_j(\hat{z})$ as in \eqref{eq:one-sided}, and a matrix-vector product with the dense factors, both $\mathcal{O}(m)$, 
where $m$ is the number of parameters of the model. It follows that an $N$-sample approximation of $\textsc{IF}_{p, \varepsilon}(z)$ requires $\mathcal{O}(mN)$ operations.
\section{Dataset Construction}
\label{sec:red_teaming}
\noindent\textbf{Overview.} Algorithm \ref{algorithm} tests whether a batch of prompts is contaminated by comparing the influence its completions exert against a known-clean reference. Evaluating this requires paired pools of malicious and benign prompts rather than single labeled examples. Existing benchmarks such as CyberSecEval \citep{bhatt2023cyberseceval, bhatt2024cyberseceval2} are not designed for this however; these benchmarks score individual completions for insecure patterns and report how often a model produces vulnerable code under non-adversarial prompting prompting, a per-sample generation metric with no notion of a clean reference pool or batch-level discrimination. We therefore construct a dedicated benchmark of $800$ Python records: a malicious pool of $400$ records whose completion targets are security-relevant implementation failures, and a benign reference pool of 
$400$ records describing ordinary coding tasks. Each pool spans two application domains, with 
$200$ records per domain. Because the domain distribution is identical in the two pools, a prompt's domain carries no information about its label, and a detector cannot succeed by recognizing subject matter alone.

Each record pairs a natural-language completion instruction with a Python skeleton containing a marked gap, indicated by one of three completion markers (\texttt{\# TODO:}, \texttt{\# TASK:}, \texttt{\# INSTRUCTION:}). The marker locates the region to be filled and does not name the vulnerability, which in the malicious pool is implied by the surrounding task framing. Figure~\ref{fig:record-examples} gives one record from each pool. Markers are assigned by a fixed three-way cyclic rule over record position within each family, independently of task content and identically in both pools, and are therefore uninformative about the label. Task names do not enter the model input.

\noindent\textbf{Provenance.}
We authored all $800$ records manually. The records themselves are original; the weakness taxonomy is drawn from the CWE catalogue and the OWASP Top~10 \citep{mitre_cwe,owasp_top10_2025}. Each malicious record began as a Python skeleton, with the completion instruction written afterwards against a target weakness from that taxonomy. The benign pool was written alongside the malicious pool, against the same two domain families. We validated the malicious pool by generating completions with several code models and confirming that the targeted weakness appears in the
generated output.

\noindent\textbf{Malicious pool.}
Weaknesses are grouped by the part of the application they target. \emph{Authentication, authorization, and session security} (\textsc{AuthSec}) covers credential and password handling, JWT and OAuth misuse, CSRF, cookie and session weaknesses, and access-control bypasses. \emph{Infrastructure, cloud, and protocol exploitation} (\textsc{InfraCloud}) covers server-side request forgery, cloud-metadata and IAM abuse, HTTP and WebSocket attacks, cache poisoning, request smuggling, and CI/CD and dependency attacks.
 
\noindent\textbf{Benign reference pool.}
Benign prompt families mirror these targets, \emph{input and data handling} (parsing, transformation, and collection utilities), \emph{authentication, identity and session utilities} (password hashing, token generation and verification, cookie configuration, role and scope checks), and \emph{infrastructure, I/O, and networking} (file I/O, logging, HTTP clients, serialization, and small Flask endpoints). Tasks are ordinary implementations with no objective to introduce a vulnerability.

\noindent\textbf{Vulnerability coverage.}
Table~\ref{tab:vuln-classes} in the Appendix breaks the malicious instruction pool into vulnerability classes with representative CWE identifiers.

\section{Experiments}
\label{sec:experiments}

\subsection{Datasets}
\label{sec:datasets}

We evaluate on the malicious and benign instruction pools described in Section 5: $200$ \textsc{AuthSec} and $200$ \textsc{InfraCloud} malicious prompts, and a benign reference pool of $400$ prompts split across the same two domains. Additionally, we use $100$ prompts sampled uniformly at random from The Stack v2 \citep{DBLP:journals/corr/abs-2402-19173} as our curvature dataset.

\begin{figure}[t]
\centering
\begin{tabular}{ccc}
\includegraphics[width=0.30\linewidth]{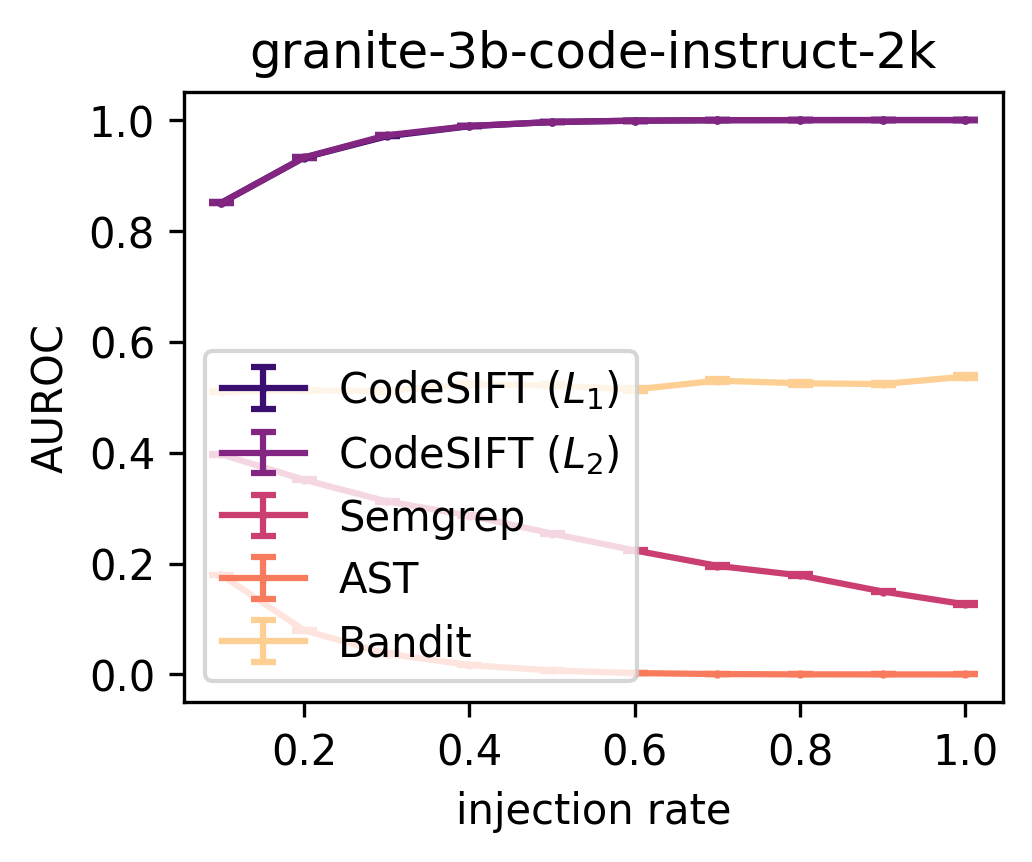} & \includegraphics[width=0.30\linewidth]{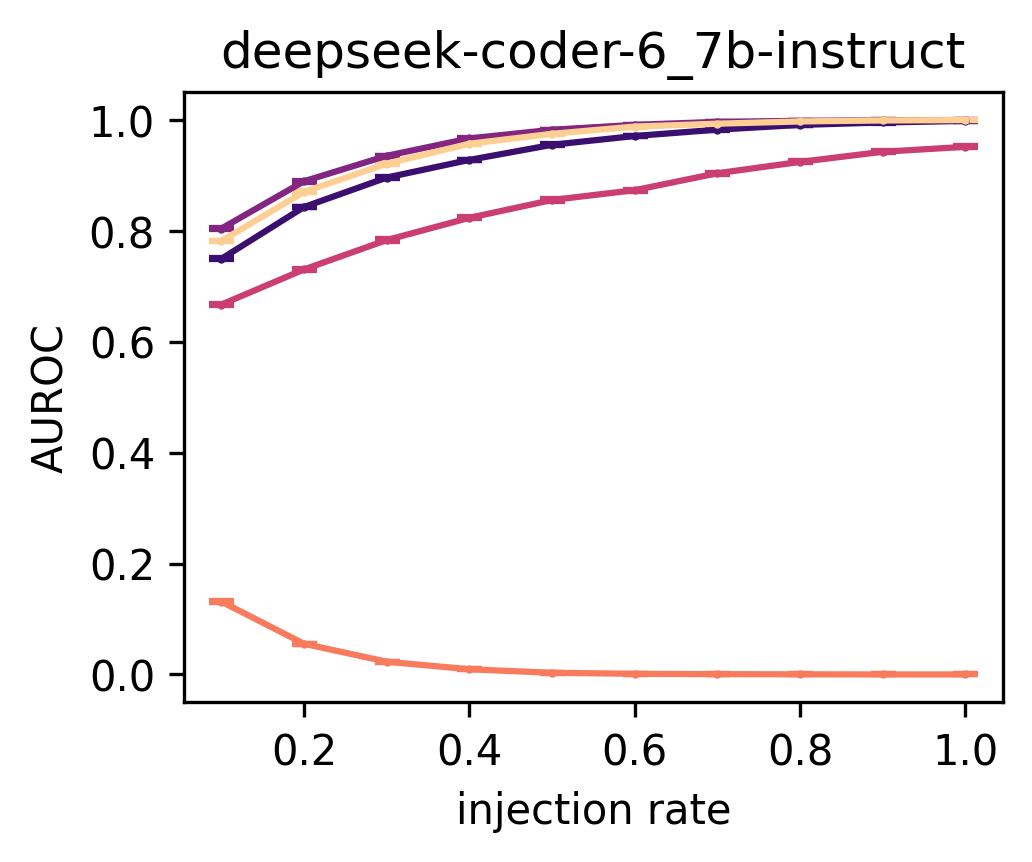} & \includegraphics[width=0.30\linewidth]{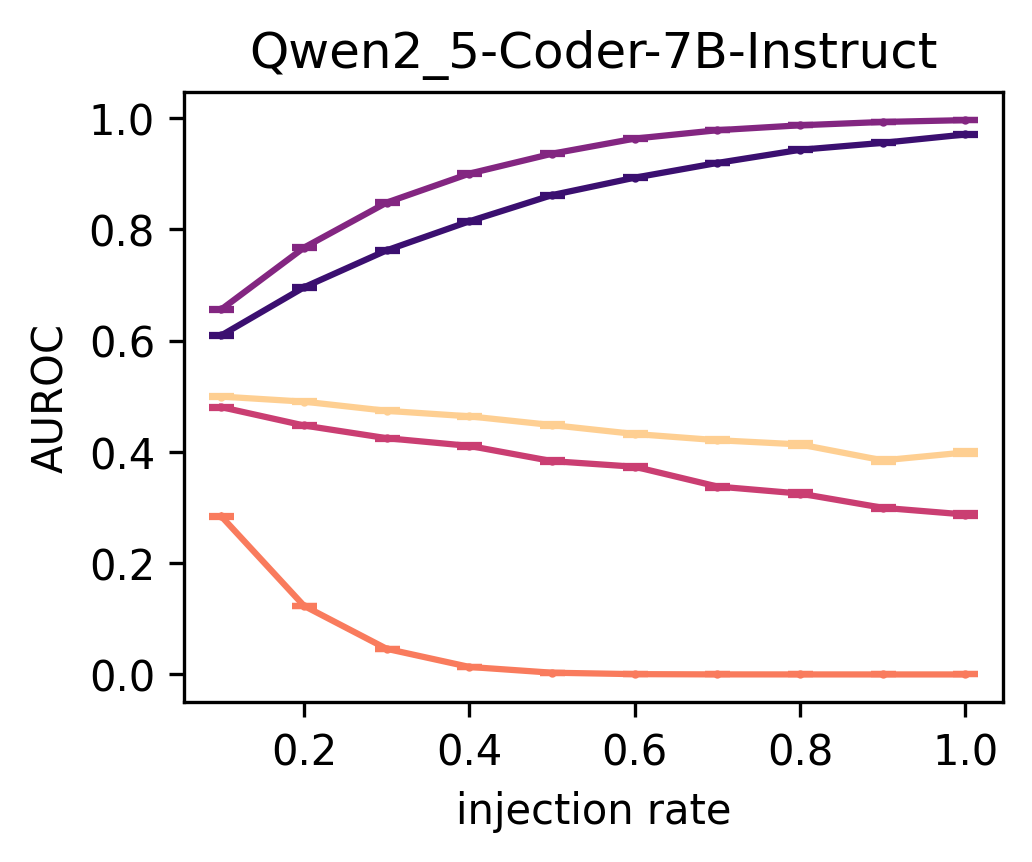} \\
\includegraphics[width=0.30\linewidth]{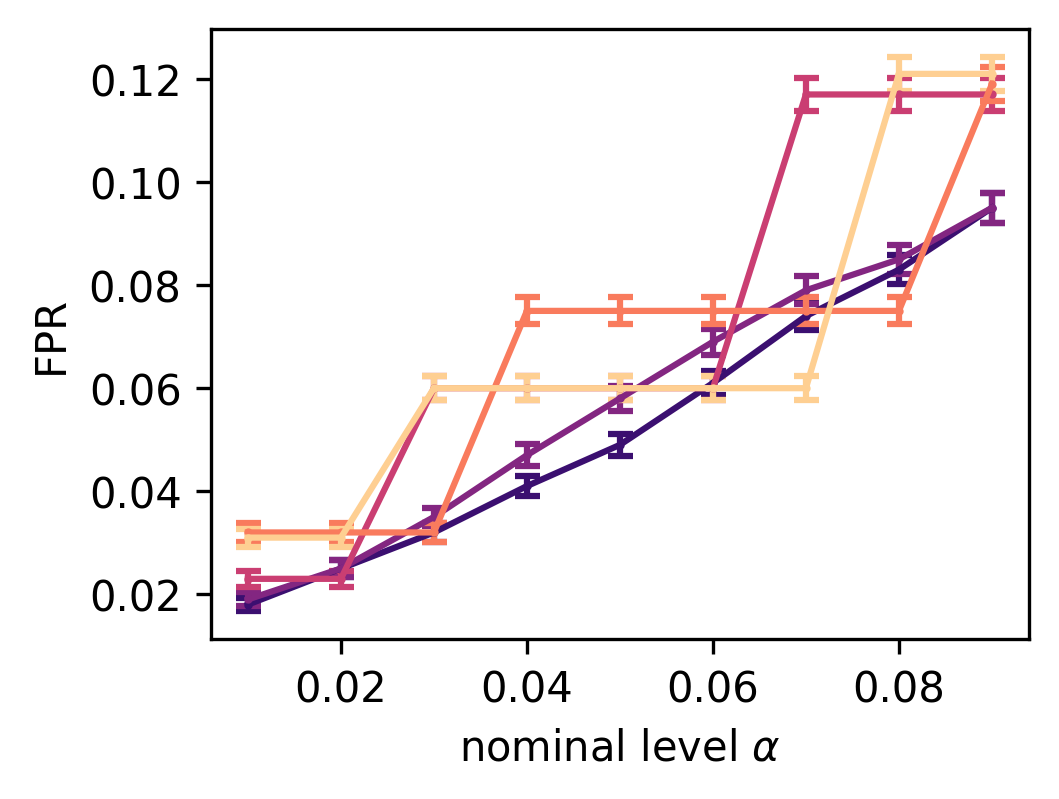} & \includegraphics[width=0.30\linewidth]{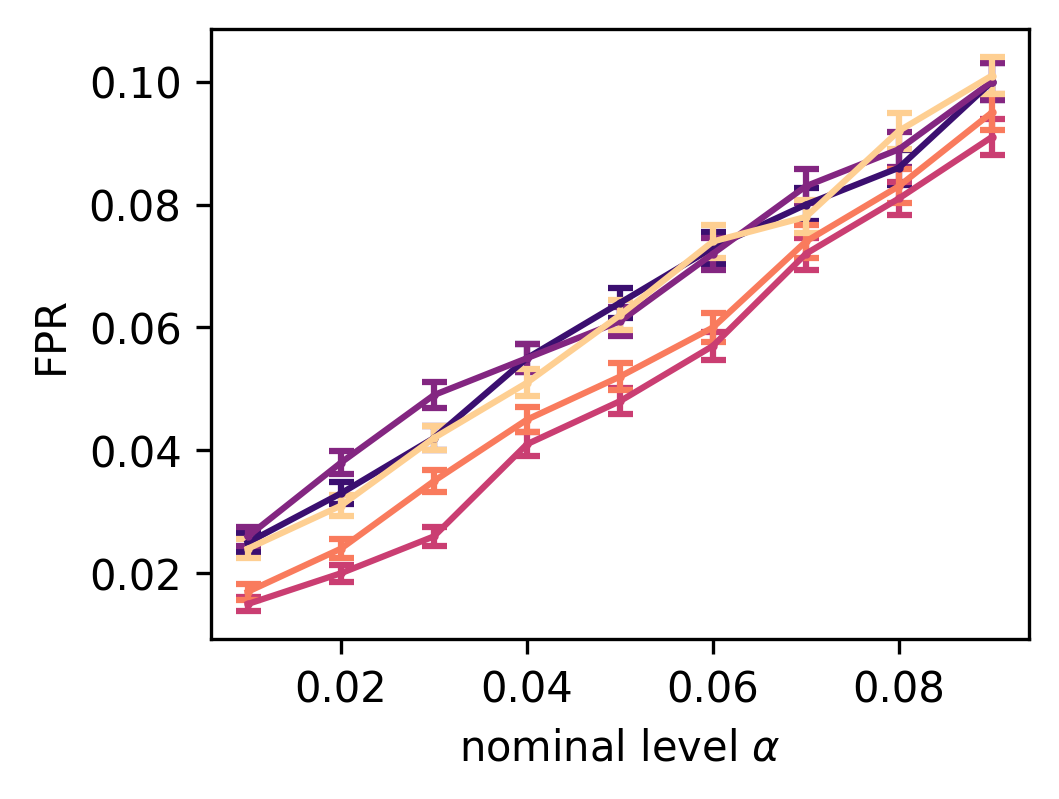} & \includegraphics[width=0.30\linewidth]{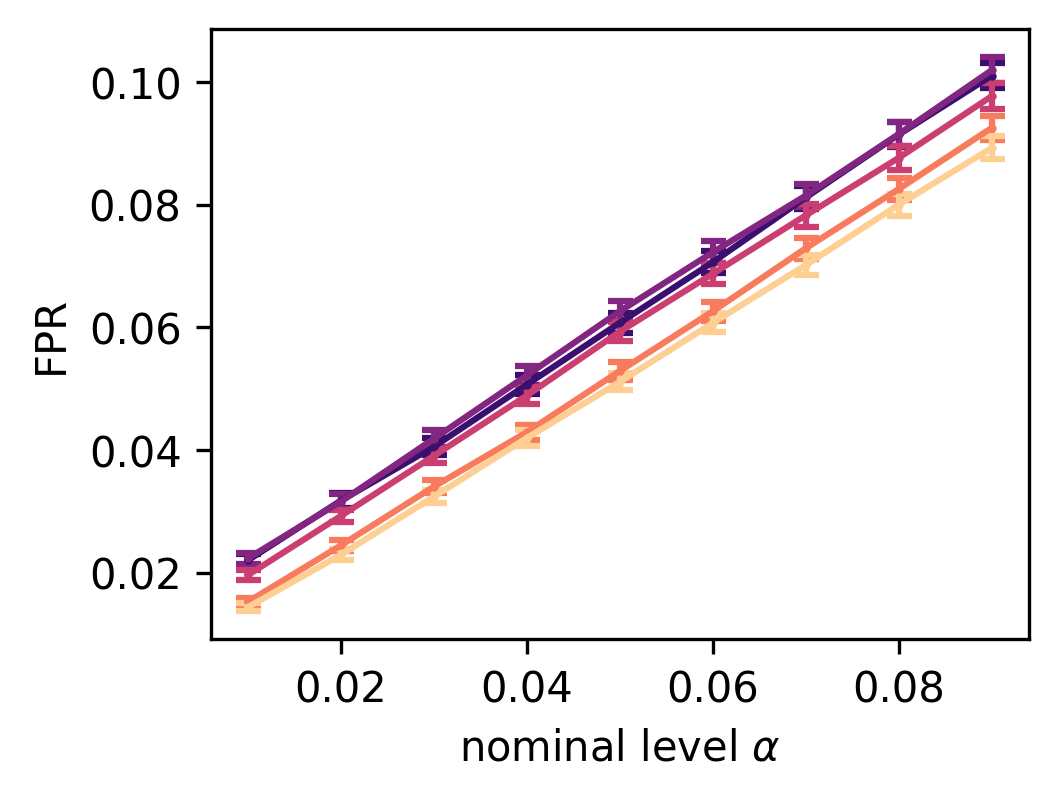} \\
\end{tabular}
\caption{Detection performance on \textsc{AuthSec} as a function of injection rate, for Granite-3B-Code-Instruct-2K, DeepSeek-Coder-6.7B-Instruct, and Qwen2.5-Coder-7B-Instruct (left to right). Top row, AUROC. Bottom row, empirical false positive rate against the nominal significance level $\alpha$. CodeSIFT approaches perfect AUROC as contamination increases and stays well calibrated across all three models, while the static baselines are inconsistent across models and often fall below chance.}
\label{fig:auth-contamination}
\end{figure}

\subsection{Models}
\label{sec:models}
\begin{enumerate}[label={$\bullet$},itemsep=0pt,topsep=0pt,leftmargin=2em]
    \item \textbf{Granite-3B-Code-Instruct-2K} \citep{granite-code}. Released by IBM as part of the Granite Code model family, it is pretrained on code and natural language and instruction-tuned for coding tasks, with a 2K-token context window. It has 3B parameters.
    \item \textbf{DeepSeek-Coder-6.7B-Instruct} \citep{deepseek-coder}. Released by DeepSeek AI, this model is pretrained on a large code-centric corpus and instruction-tuned for code generation and completion tasks. It has 6.7B parameters.
    \item \textbf{Qwen2.5-Coder-7B-Instruct} \citep{hui2024qwen2,qwen2}. Released by Alibaba as part of the Qwen2.5-Coder family, it is pretrained on a mixture of code and general text and instruction-tuned for coding tasks. It has 7B parameters.
\end{enumerate}
\subsection{Baselines}
\label{sec:baselines}

\begin{figure}[t]
\centering
\begin{tabular}{ccc}
\includegraphics[width=0.30\linewidth]{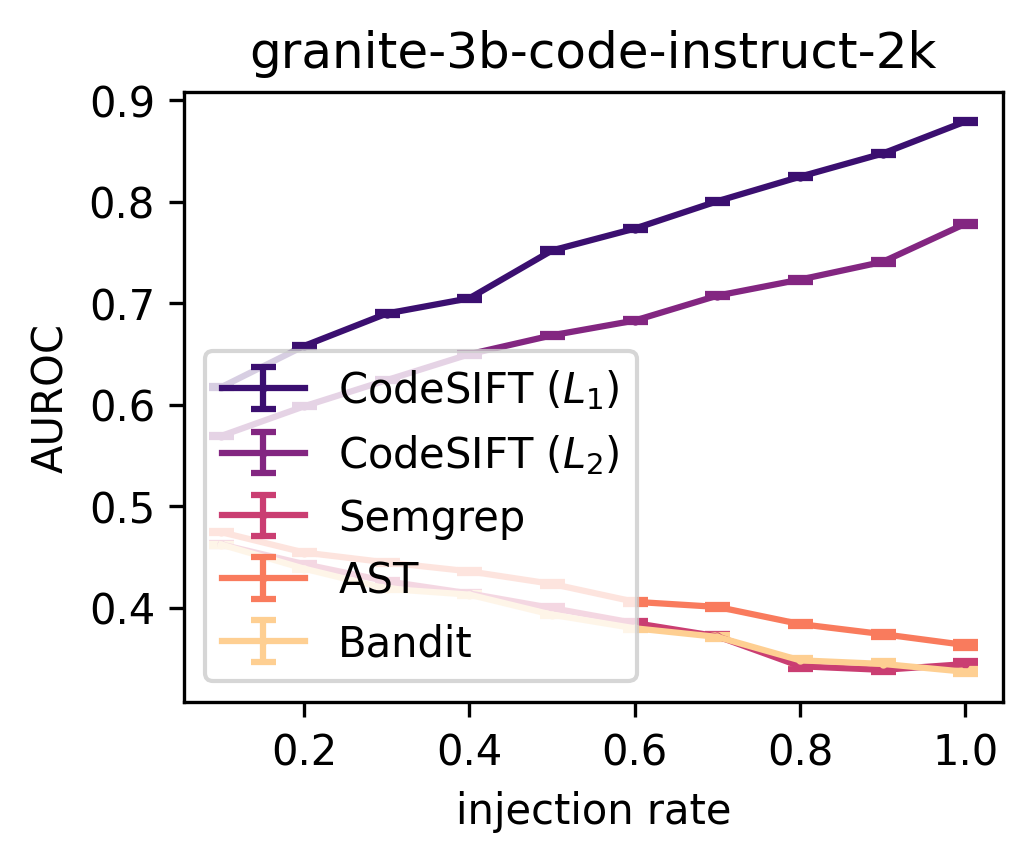} & \includegraphics[width=0.30\linewidth]{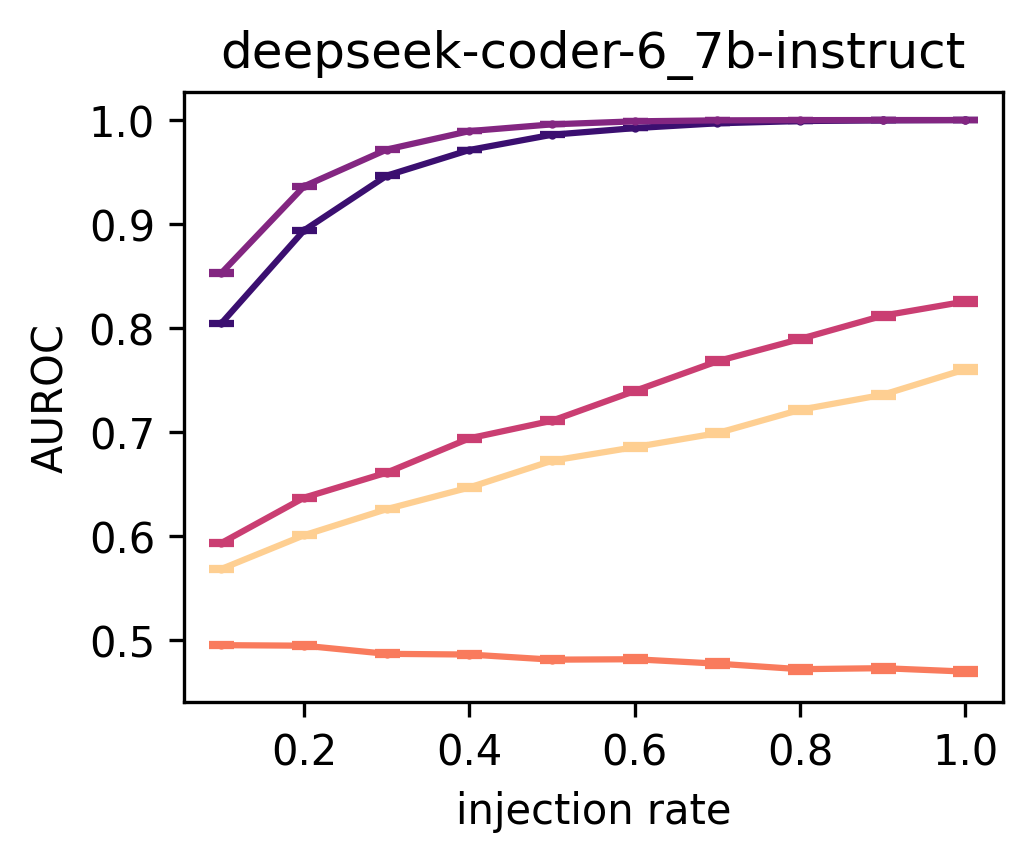} & \includegraphics[width=0.30\linewidth]{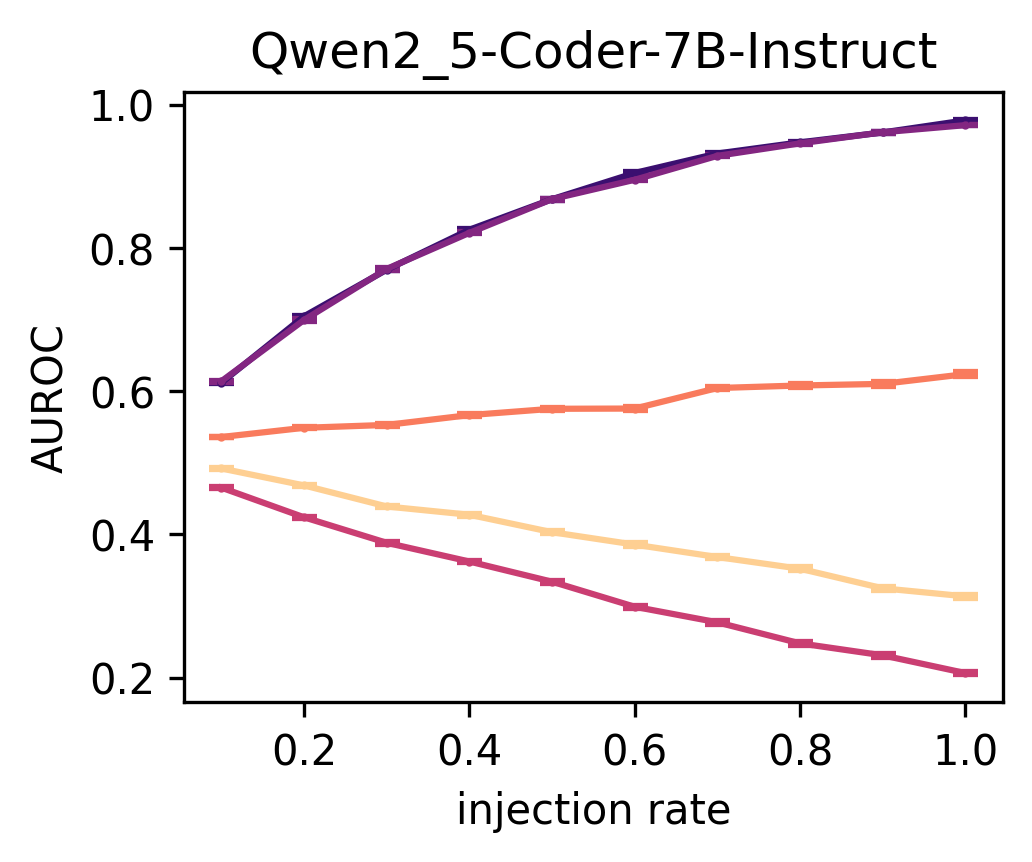} \\
\includegraphics[width=0.30\linewidth]{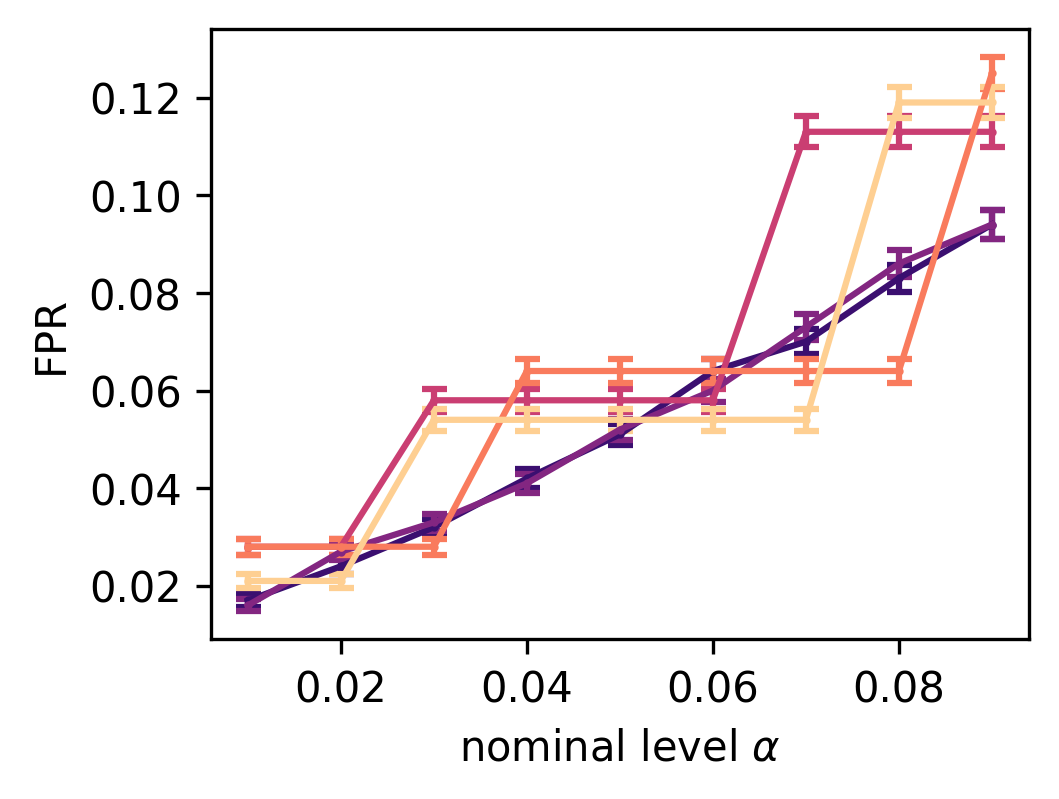} & \includegraphics[width=0.30\linewidth]{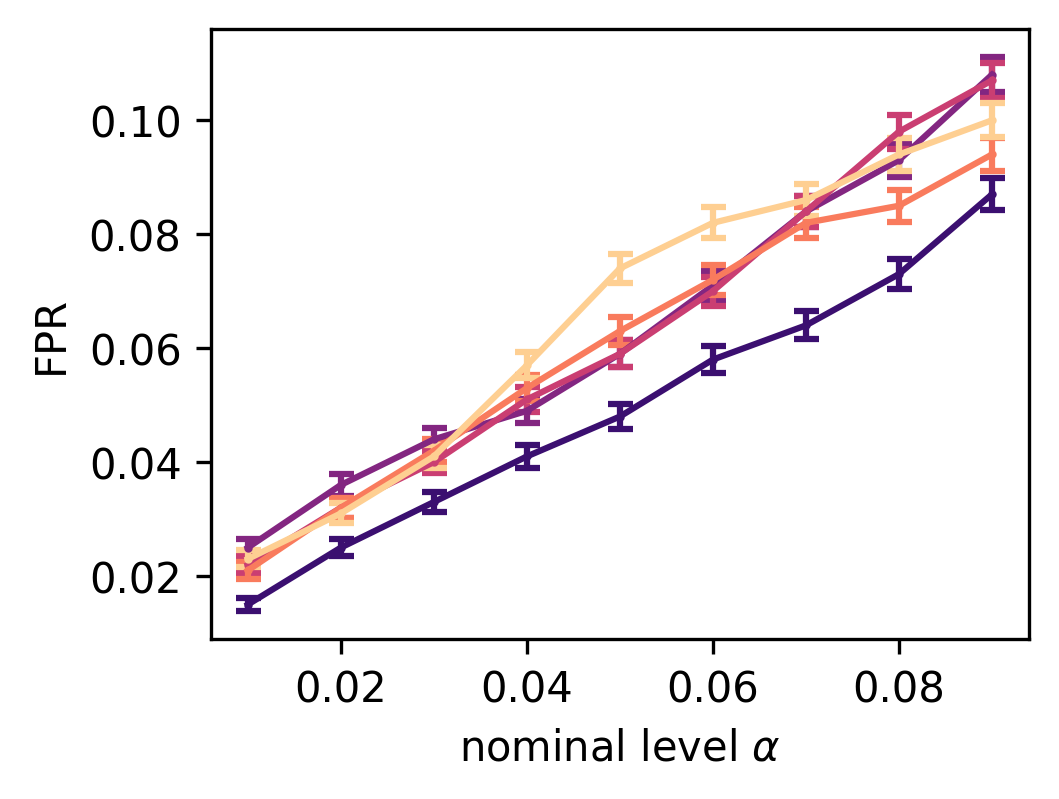} & \includegraphics[width=0.30\linewidth]{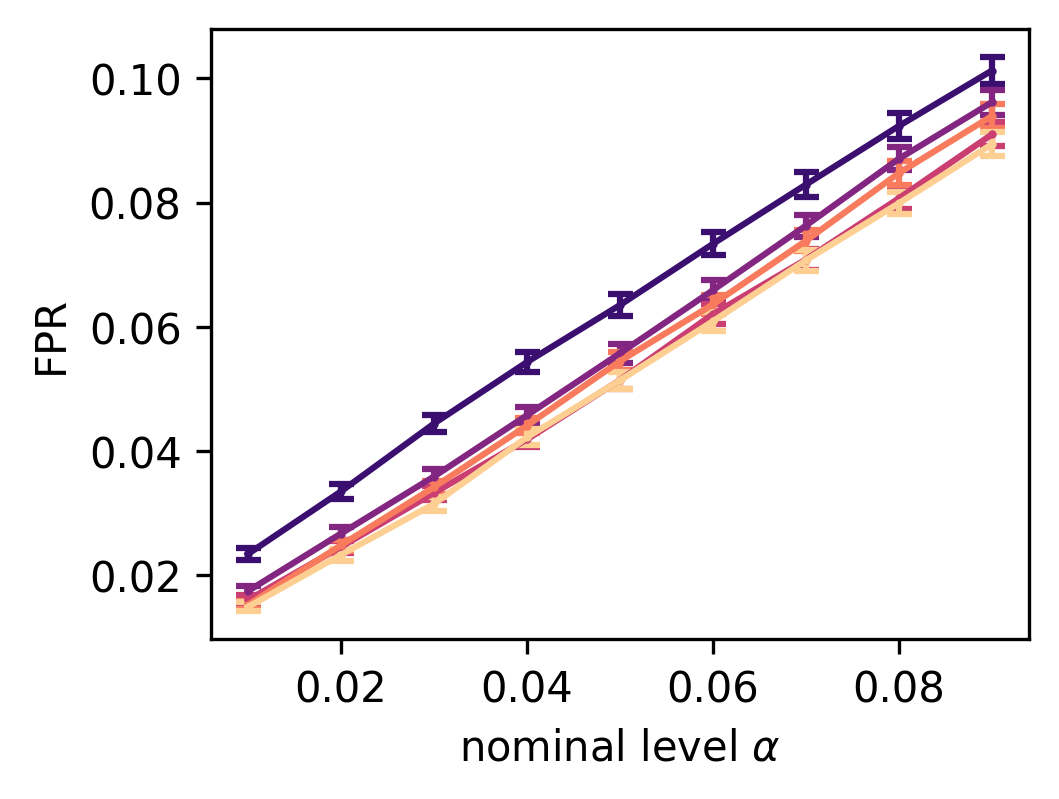} \\
\end{tabular}
\caption{Detection performance on \textsc{InfraCloud} as a function of injection rate, for Granite-3B-Code-Instruct-2K, DeepSeek-Coder-6.7B-Instruct, and Qwen2.5-Coder-7B-Instruct (left to right). Top row, AUROC. Bottom row, empirical false positive rate against the nominal significance level $\alpha$. \textsc{InfraCloud} is harder than \textsc{AuthSec}, so CodeSIFT's AUROC ceiling is lower, but it still improves monotonically with contamination and remains calibrated, while the baselines flatten or decline below chance.}
\label{fig:infrastructure-contamination}
\end{figure}

We compare against three static and pattern-based detectors, each producing a binary flag per generated prompt (positive if the detector reports at least one finding). We average each detector's flag over the candidate set $\mathtt D$ and reference set $\mathtt D^*$, then run the same statistical test used for CodeSIFT (a $t$-test as in Algorithm \ref{algorithm}) on these values to test whether $\mathtt D$'s flag rate is significantly elevated. This holds the testing procedure fixed across detectors, varying only the underlying signal.

\begin{enumerate}[label={$\bullet$},itemsep=0pt,topsep=0pt,leftmargin=2em]
    \item \textbf{Bandit} \citep{banditpycqa}. A static analysis security linter for Python that inspects the abstract syntax tree for common misuse patterns, such as use of \texttt{eval}, weak cryptographic primitives, hardcoded passwords, and insecure subprocess calls.
    \item \textbf{Semgrep} \citep{semgrep2024}. A pattern-based static analysis tool run with publicly available, off-the-shelf rule packs for Python web applications (covering Flask and Django), drawn from Semgrep's open-source rule registry rather than written by us. These rules encode generic secure-coding patterns, e.g., injection, insecure deserialization, weak cryptography, CSRF handling, and are not tailored to \textsc{AuthSec} or \textsc{InfraCloud}.
    \item \textbf{AST Taint Tracker}. A custom static taint-analysis baseline that tracks data flow, via Python's \texttt{ast} module, from Flask request sources (\texttt{request.args}, \texttt{request.json}, \texttt{request.form}, and related attributes) to security-sensitive sinks, including SQL execution, command execution, deserialization, SSRF, XSS, and NoSQL or LDAP queries.
\end{enumerate}

To the best of our knowledge, no established tool targets the threat detection problem studied in this work, so two baselines are repurposed general-purpose linters and the third we implement ourselves. All three are built agnostically: we assume the detector has no advance knowledge of the injected attack, and avoid tuning any baseline to \textsc{AuthSec}'s or \textsc{InfraCloud}'s specific vulnerability classes, mirroring CodeSIFT's own no-threat-model-knowledge assumption (Section \ref{sec:algorithm}).

\subsection{Compute}
All experiments are run on a single node with four NVIDIA RTX A6000 GPUs
(48GB each, driver 570.211.01, CUDA 12.8). Granite-3B-Code-Instruct-2K and
DeepSeek-Coder-6.7B-Instruct each run on a single GPU. Qwen2.5-Coder-7B-Instruct
is split across two GPUs to accommodate its larger memory footprint. Because dense Kronecker factors can be memory-intensive and numerically delicate in low-precision training \citep{lin2024structured}, gradient outer products are formed in \texttt{bfloat16} while the accumulated factors $\boldsymbol{A}_j$ and $\boldsymbol{G}_j$ are stored in \texttt{float32}. When multiple GPUs are available, per-projection factors are distributed across devices by assigning projection $j$ to device $j \bmod N_{\mathrm{gpu}}$ \citep{lu2026runtimeorchestratedsecondorderoptimizationscalable}, spreading the curvature state across devices rather than pinning it to a single GPU.

The hyperparameters for the experiments are presented in Appendix \ref{sec:hyperparameters}.
\subsection{Results}
\label{sec:results}

Figures~\ref{fig:auth-contamination} and~\ref{fig:infrastructure-contamination} report detection performance on \textsc{AuthSec} and \textsc{InfraCloud} respectively, as a function of the injection rate, for CodeSIFT under the $\ell_1$ and $\ell_2$ norm and for the three static baselines, across all three models.

\noindent\textbf{Detection accuracy.} CodeSIFT improves monotonically with injection rate in every model and dataset combination we test. On \textsc{AuthSec}, AUROC already exceeds 0.6 to 0.85 at the lowest injection rate we evaluate, and saturates to essentially 1.0 once roughly 40 to 50 percent of the candidate set is malicious. On \textsc{InfraCloud} the task is harder, AUROC rises more gradually and peaks between 0.88 and 0.98 depending on the model, which is consistent with \textsc{InfraCloud} spanning a broader and more heterogeneous set of vulnerability classes than \textsc{AuthSec}. The baselines show no comparably reliable signal. Their performance is inconsistent across models rather than merely weaker, Bandit is close to CodeSIFT on DeepSeek-Coder ($\approx$0.76 to 1.0 AUROC) but stays near chance on Granite and Qwen, while Semgrep and the AST Taint Tracker frequently fall below 0.5 AUROC as the injection rate grows, meaning their flags become anti-correlated with the true label rather than merely uninformative. We take this as evidence that static, pattern-based detectors latch onto surface regularities that do not transfer across models, whereas the parameter-space influence signal that CodeSIFT relies on does.

\noindent\textbf{False positive control.} We check calibration by comparing the empirical false positive rate of the statistical test against the nominal significance level $\alpha$ under the null hypothesis, that is, when the candidate set is not contaminated. For CodeSIFT, the empirical FPR tracks $\alpha$ closely across the full range we test (0.01 to 0.09) for all three models on both datasets, indicating the test is well calibrated rather than systematically over- or under-rejecting. The baselines are only as well calibrated as CodeSIFT on some model and dataset combinations, for example DeepSeek-Coder on \textsc{AuthSec}, but on others their FPR curves show a step-like pattern that departs from the diagonal, which we attribute to their binary per-sample flag admitting only a coarse set of achievable proportions, unlike CodeSIFT's continuous influence score.

\noindent\textbf{Additional experiments.} Appendix~\ref{sec:appendix-sweeps} reports the same metrics as a function of candidate set size and number of sampled completions per prompt, and jointly across pairs of these axes, for every detector, model, and dataset combination.

\section{Discussion} 
Across both datasets and all three models tested, CodeSIFT detects contaminated prompt batches without any assumption on the threat model, reaching AUROC rises monotonically with the fraction of malicious prompts in a batch, reaching up to 
$0.98$ once contamination is moderate to high, and remains well calibrated throughout (Section \ref{sec:experiments}). Performance is weaker, though still generally above chance, in the hardest regime of small candidate batches (Appendix~\ref{sec:appendix-sweeps}). Meanwhile, the static baselines are inconsistent across models and frequently uninformative or worse. This indicates that parameter-space influence carries a detection signal that transfers across model families in a way surface-level pattern matching does not, and does so without requiring the vulnerability class to be known in advance. We note a few directions the current results do not yet cover. CodeSIFT is designed as a batch-level test; extending it to localize individual malicious prompts within a batch is a natural next step. Our evaluation spans models up to $7$B parameters, consistent with the compute budget for this study, and prior work on influence functions at up to $52$B parameters \citep{grosse2023llm} suggests the approach should scale further. Similarly, extending the benchmark beyond Python and the vulnerability families studied here would help establish the generality of the method across languages and attack surfaces.

\section{Acknowledgments}
The authors gratefully acknowledge James Dixon for his assistance with code development and implementation.

\bibliographystyle{apalike}
\bibliography{references}
\newpage
\renewcommand{\thesection}{\Alph{section}}
\setcounter{section}{0}
\noindent {\LARGE\textbf{Appendix}}
\section{Missing Proofs}
\label{appendix: missing_proof}
\begin{proposition}
\label{prop:pbrf-linearized}
Assume the loss $\mathcal L(\cdot,\cdot)$ is convex as a function of the network outputs. Then, it holds
\begin{equation}
    \hat{r}_{z,\mathrm{lin}}(\varepsilon)
    \approx
    \arg\min_{\theta \in \mathbb{R}^d}
    \left[
    \frac{1}{N}\sum_{i=1}^{N}
    D_{\mathcal L}(z^{(i)}, \theta, \hat{\theta})
    +
    \varepsilon\,\mathcal L(\theta, x^{(i)})
    +
    \frac{\lambda}{2}
    \|\theta-\hat{\theta}\|^2
    \right],
\end{equation}
where $D_{\mathcal L}^{\quad}$ is the Bregman divergence of the loss, defined as
\[
D_{\mathcal L}(z^{(i)}, \theta, \hat{\theta}) = \mathcal L(\theta, z^{(i)}) - \mathcal L(\hat{\theta}, z^{(i)}) - \nabla \mathcal L(\hat{\theta}, z^{(i)})^\top (f(z^{(i)}, \theta) - f(z^{(i)}, \hat{\theta})),\footnote{With a mild abuse of notation, in this equation we denote with $f(z^{(i)}, \theta)$ and $f(z^{(i)}, \hat{\theta})$ the networks' predicted output distribution.}
\]
with the gradient $\nabla \mathcal L(\hat{\theta}, z)$ taken w.r.t. the network's output distributions.
\end{proposition}

\begin{proof}
In order to simplify the notation, we denote with $\hat{y}^{(i)}$ the activations of $f(\cdot, \hat{\theta})$ for point $z^{(i)}$. Note that $\hat{y}^{(i)}$ is fixed for input $z^{(i)}$, even if the output model is stochastic. Similarly, we denote with $y^{(i)}$ the activations of $f(\cdot, \theta)$ for point $z^{(i)}$. Consider the second-order quadratic expansion of the loss $\mathcal L$ around $\hat{y}^{(i)}$, defined as
\begin{equation}
\label{eq:loss-quad}
\mathcal L_{\mathrm{quad}}(y)
  \coloneqq \mathcal L(\hat{y}^{(i)})
    + \nabla_y \mathcal L(\hat{y}^{(i)})^\top (y^{(i)} - \hat{y}^{(i)})
    + (y^{(i)} - \hat{y}^{(i)})^\top \nabla_y^2 \mathcal L(\hat{y}^{(i)})(y^{(i)} - \hat{y}^{(i)}).
\end{equation}
Consider the expected linearized network output around $\hat{\theta}$ of this function, defined as
\begin{equation*}
f(z^{(i)}, \theta) \approx f(z^{(i)}, \hat{\theta}) + J_{\theta}f(z^{(i)}, \hat{\theta}) (\theta - \hat{\theta}).
\end{equation*}
Note that it holds
\begin{equation}
\label{eq:output-lin}
y^{(i)} - \hat{y}^{(i)} = f(z^{(i)}, \theta) - f(z^{(i)}, \hat{\theta}) \approx J_{\theta}f(z^{(i)}, \hat{\theta}) (\theta - \hat{\theta}).
\end{equation}
\noindent\textbf{Step 1: The Bregman term reduces to a quadratic form.}
Computing the Bregman divergence of the quadratic approximation of the loss as in \eqref{eq:loss-quad}, it holds
\begin{align}
D_{\mathcal L_\mathrm{quad}}(y^{(i)}, \hat{y}^{(i)}) & = \mathcal L_\mathrm{quad} (y^{(i)}) - \mathcal L_\mathrm{quad} (\hat{y}^{(i)}) - \nabla_{y}\mathcal L_\mathrm{quad}(\hat{y}^{(i)})^\top (y^{(i)} - \hat{y}^{(i)})\nonumber \\
  & = \mathcal L(\hat{y}^{(i)})    + \nabla_y \mathcal L(y^{(i)})^\top (y^{(i)} - \hat{y}^{(i)})    + (y^{(i)} - \hat{y}^{(i)})^\top \nabla_y^2 \mathcal L(\hat{y}^{(i)})(y^{(i)} - \hat{y}^{(i)}) - \mathcal L(\hat{y}^{(i)})\nonumber  \\
    & \quad - \nabla_{y}\mathcal L_\mathrm{quad}(\hat{y}^{(i)})^\top (y^{(i)} - \hat{y}^{(i)})\nonumber \\
 & = \mathcal L(\hat{y}^{(i)})
    + \nabla_y \mathcal L(\hat{y}^{(i)})^\top (y^{(i)} - \hat{y}^{(i)})
    + (y^{(i)} - \hat{y}^{(i)})^\top \nabla_y^2 \mathcal L(\hat{y}^{(i)})(y^{(i)} - \hat{y}^{(i)}) - \mathcal L(\hat{y}^{(i)})\nonumber \\
    & \quad - \nabla_{y}\mathcal L(\hat{y}^{(i)})^\top (y^{(i)} - \hat{y}^{(i)})\nonumber \\
  & = (y^{(i)} - \hat{y}^{(i)})^\top \nabla_y^2 \mathcal L(\hat{y}^{(i)})(y^{(i)} - \hat{y}^{(i)}) \label{last-Bregman}
\end{align}
Substituting the linearized output \eqref{eq:output-lin} in \eqref{last-Bregman} gives
\begin{equation}
\label{eq:bregman-quad}
D_{\mathcal L_\mathrm{quad}}(y^{(i)}, \hat{y}^{(i)}) 
  \approx (\theta - \hat{\theta})^\top J_{\theta}f(z^{(i)}, \hat{\theta})^\top
    \nabla_y^2 \mathcal L(\hat{y}^{(i)})\, J_{\theta}f(z^{(i)}, \hat{\theta}) (\theta - \hat{\theta}).
\end{equation}

\noindent\textbf{Step 2: Stationarity condition.}
Averaging \eqref{eq:bregman-quad} over $i = 1, \dots, N$ and substituting into the linearized
PBRF objective, we differentiate with respect to $\theta$ and set
the gradient to zero:
\begin{align}
0 &= \frac{1}{N}\sum_{i=1}^N J_{\theta}f(z^{(i)}, \hat{\theta})^\top
    \nabla_y^2 \mathcal L(\hat{y}^{(i)})\, J_{\theta}f(z^{(i)}, \hat{\theta}) (\theta - \hat{\theta})\;+\; \epsilon \nabla_\theta \mathcal L(\hat{\theta},z)\,
      \;+\; \lambda (\theta - \hat{\theta}) \nonumber \\
  &= G_{\hat{\theta}}(\theta - \hat{\theta})
      \;+\; \epsilon \nabla_\theta \mathcal L(\hat{\theta},z)
      \;+\; \lambda (\theta - \hat{\theta}), \label{eq:GNH}
\end{align}
where $G_{\hat{\theta}}$ is an approximation of the Gauss-Newton Hessian at $\hat{\theta}$.

\noindent\textbf{Step 3: Solve for $\theta$.}
Rearranging \eqref{eq:GNH},
\begin{equation}
\bigl(G_{\hat{\theta}} + \lambda I\bigr)(\theta - \hat{\theta}) = - \epsilon \nabla_\theta \mathcal L(\hat{\theta},z),
\end{equation}
and since $\mathcal L$ is convex as a function of the network outputs, then $G_{\hat{\theta}} \succeq 0$; adding $\lambda I$ with $\lambda > 0$ guarantees invertibility. Solving
for $\theta$ therefore yields the unique minimizer
\begin{equation}
\hat{r}_{z,\mathrm{lin}}(\varepsilon)
  = \hat{\theta} - \bigl(G_{\hat{\theta}} + \lambda I\bigr)^{-1} \nabla_\theta \mathcal L(\hat{\theta},z)\,\epsilon,
\end{equation}
as claimed.
\end{proof}
\section{Datasets vulnerability coverage}
Table~\ref{tab:vuln-classes} groups the malicious samples by vulnerability class, with representative CWE IDs for each.
\begin{table}[p]
    \centering
    \scriptsize
    \caption{Vulnerability classes in the malicious pool. Classes are listed in
    descending order of frequency within each family; each family totals 200
    records. CWE identifiers are \emph{representative}: each names the dominant
    weakness in its class, not a per-record label, and several classes span more
    than one entry. Rows marked \texttt{--} aggregate a long tail of weaknesses
    that do not map cleanly to a single identifier.}
    \label{tab:vuln-classes}
    \begin{tabular}{@{}p{0.74\linewidth}lr@{}}
                \toprule
        \textbf{Vulnerability class} & \textbf{Rep. CWE} & \textbf{Records} \\
        \midrule
        \multicolumn{3}{@{}l}{\emph{Authentication, authorisation, and session security}} \\[1pt]
        \quad Credential \& secret exposure & CWE-200 & 23 \\
        \quad Broken access control / IDOR & CWE-639 & 22 \\
        \quad Federated auth / SSO flaws & CWE-287 & 19 \\
        \quad Session management & CWE-384 & 19 \\
        \quad JWT verification flaws & CWE-347 & 17 \\
        \quad Business-logic flaws & CWE-841 & 17 \\
        \quad Weak cryptography \& randomness & CWE-327 & 16 \\
        \quad Cookie security & CWE-1004 & 12 \\
        \quad Brute-force \& MFA bypass & CWE-307 & 9 \\
        \quad Cross-site request forgery & CWE-352 & 8 \\
        \quad Password storage and hashing & CWE-916 & 7 \\
        \quad Timing \& enumeration side channels & CWE-208 & 7 \\
        \quad Password reset flows & CWE-640 & 6 \\
        \quad Password policy & CWE-521 & 6 \\
        \quad Other authentication failures & -- & 6 \\
        \quad Race conditions & CWE-362 & 6 \\
        \quad\textbf{Subtotal} & & \textbf{200} \\
        \midrule
        \multicolumn{3}{@{}l}{\emph{Infrastructure, cloud, and protocol exploitation}} \\[1pt]
        \quad Information disclosure & CWE-538 & 26 \\
        \quad Other infrastructure failures & -- & 21 \\
        \quad Server-side request forgery & CWE-918 & 18 \\
        \quad HTTP protocol abuse \& smuggling & CWE-444 & 17 \\
        \quad CI/CD \& supply-chain attacks & CWE-1357 & 15 \\
        \quad Cache poisoning \& deception & CWE-349 & 15 \\
        \quad Security-header \& framing flaws & CWE-1021 & 12 \\
        \quad TOCTOU \& file-race conditions & CWE-367 & 12 \\
        \quad Denial of service & CWE-400 & 10 \\
        \quad Container \& orchestration escape & CWE-269 & 9 \\
        \quad Injection into cloud service APIs & CWE-943 & 9 \\
        \quad Protocol downgrade \& confusion & CWE-757 & 9 \\
        \quad Cloud credential \& secret leakage & CWE-522 & 7 \\
        \quad DNS \& network-level attacks & CWE-346 & 7 \\
        \quad WebSocket security & CWE-1385 & 5 \\
        \quad CORS misconfiguration & CWE-942 & 4 \\
        \quad Cloud permission misconfiguration & CWE-732 & 4 \\
        \quad\textbf{Subtotal} & & \textbf{200} \\
        \bottomrule
    \end{tabular}
\end{table}

\section{Hyperparameters for the Main Experiments}
\label{sec:hyperparameters}
Unless otherwise noted, all experiments use the following fixed hyperparameters across models,
datasets, and injection rates; Appendix~\ref{sec:appendix-sweeps} instead reports the effect of sweeping candidate-set size,
sample count, and injection rate directly as figure axes.

\begin{table}[h]
\centering
\caption{Hyperparameters used across all experiments (Algorithm~\ref{algorithm} and Section~\ref{sec:preliminaries} notation in
parentheses).}
\label{tab:hyperparams}
\begin{tabular}{ll}
\toprule
Hyperparameter & Value \\
\midrule
Upweighting factor ($\varepsilon$) & $1\times10^{-3}$ \\
Samples per prompt ($N$) & 25 \\
Base K-FAC damping & $0.1$ \\
Relative damping coefficient (rel\_eps) & $0.1$ \\
Damping floor ($\gamma$, Granite-3B-Code-Instruct-2K, DeepSeek-Coder-6.7B-Instruct) & $1\times10^{-3}$ \\
Damping floor ($\gamma$, Qwen2.5-Coder-7B-Instruct) & $1\times10^{-2}$ \\
Significance threshold ($\alpha$, detection decision) & $0.01$ \\
Significance threshold sweep ($\alpha$, FPR calibration) & $0.01$--$0.10$ \\
\bottomrule
\end{tabular}
\end{table}

We fix $\varepsilon = 10^{-3}$ and average $N=25$ sampled responses per prompt when estimating
$\text{IF}_{p,\varepsilon}(z)$. K-FAC factors are damped with a base coefficient of $0.1$, refined
per-projection via a relative damping term of $0.1$ and a damping floor $\gamma$; we use
$\gamma = 10^{-3}$ for Granite-3B-Code-Instruct-2K and DeepSeek-Coder-6.7B-Instruct, and
$\gamma = 10^{-2}$ for Qwen2.5-Coder-7B-Instruct, reflecting differences in curvature scale across
architectures. The headline detection results use a fixed significance threshold of
$\alpha = 0.01$. To assess calibration, Section~6.5 additionally sweeps the nominal significance
level from $0.01$ to $0.10$ and reports the resulting empirical false positive rate.
\section{Additional experimental results}
\label{sec:appendix-sweeps}

\subsection{Overview across dataset size and sample count}
Figures~\ref{fig:auth-standalone1}-\ref{fig:auth-standalone2} and~\ref{fig:infrastructure-standalone1}-\ref{fig:infrastructure-standalone2} plot AUROC as a function of dataset size, and separately as a function of the number of sampled completions per prompt, with all five detectors overlaid on the same axes. On AuthSec, CodeSIFT reaches near-ceiling AUROC with as few as 20 to 30 prompts in the candidate set, and its accuracy is essentially unaffected by how many completions are sampled per prompt, as few as 5 samples already suffice. On InfraCloud the same pattern holds at a lower ceiling, consistent with the harder overall task discussed in Section~\ref{sec:results}. The three static baselines are flat or slowly declining along both axes, they do not benefit from a larger candidate set or from additional sampled completions the way CodeSIFT does.

\begin{figure}[ht]
\centering
\begin{tabular}{ccc}
\includegraphics[width=0.30\linewidth]{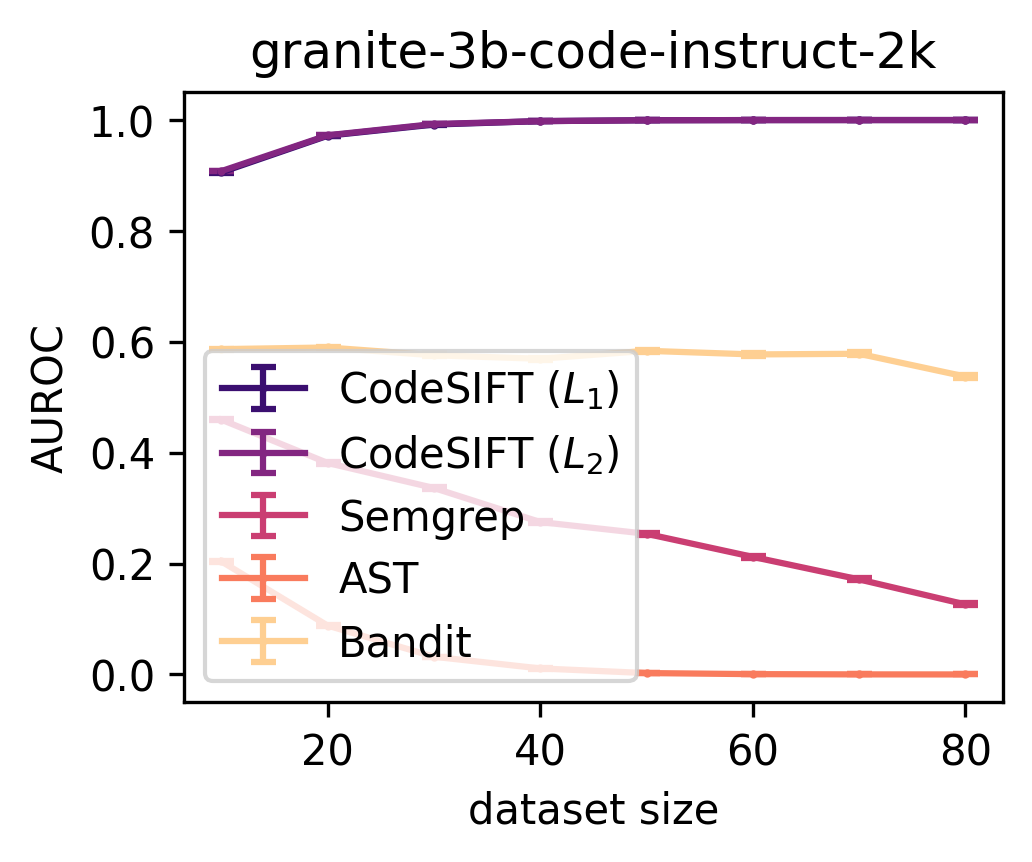} & \includegraphics[width=0.30\linewidth]{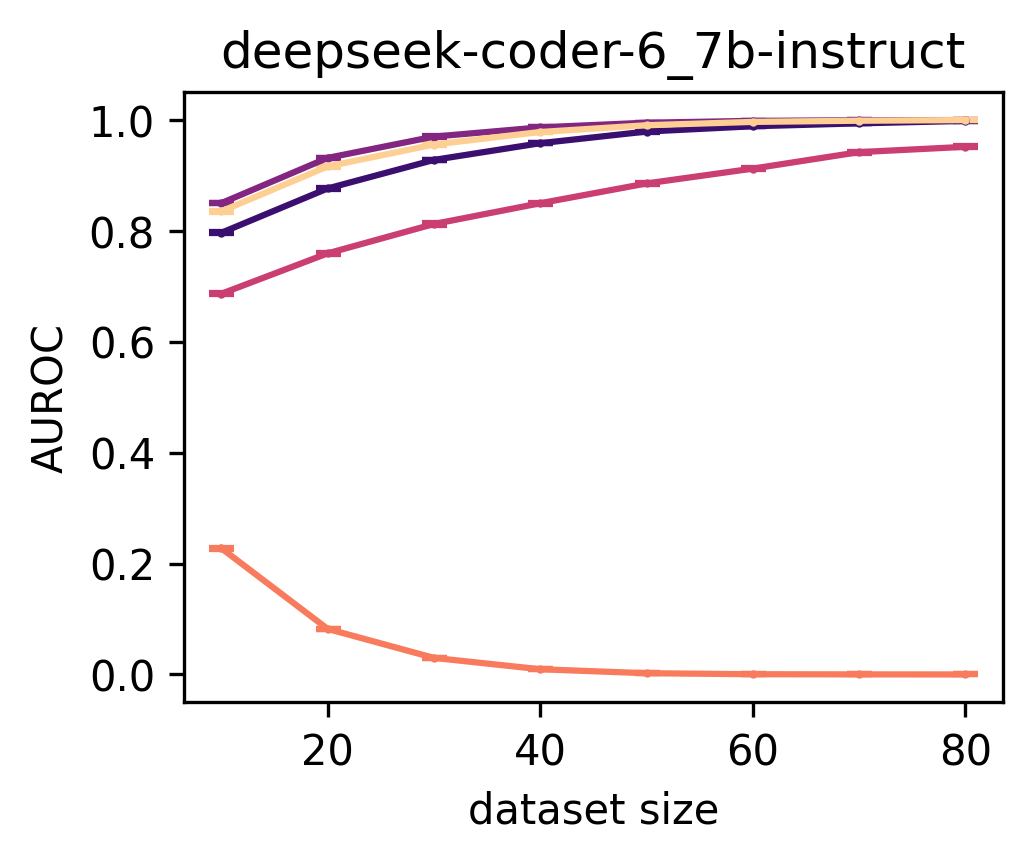} & \includegraphics[width=0.30\linewidth]{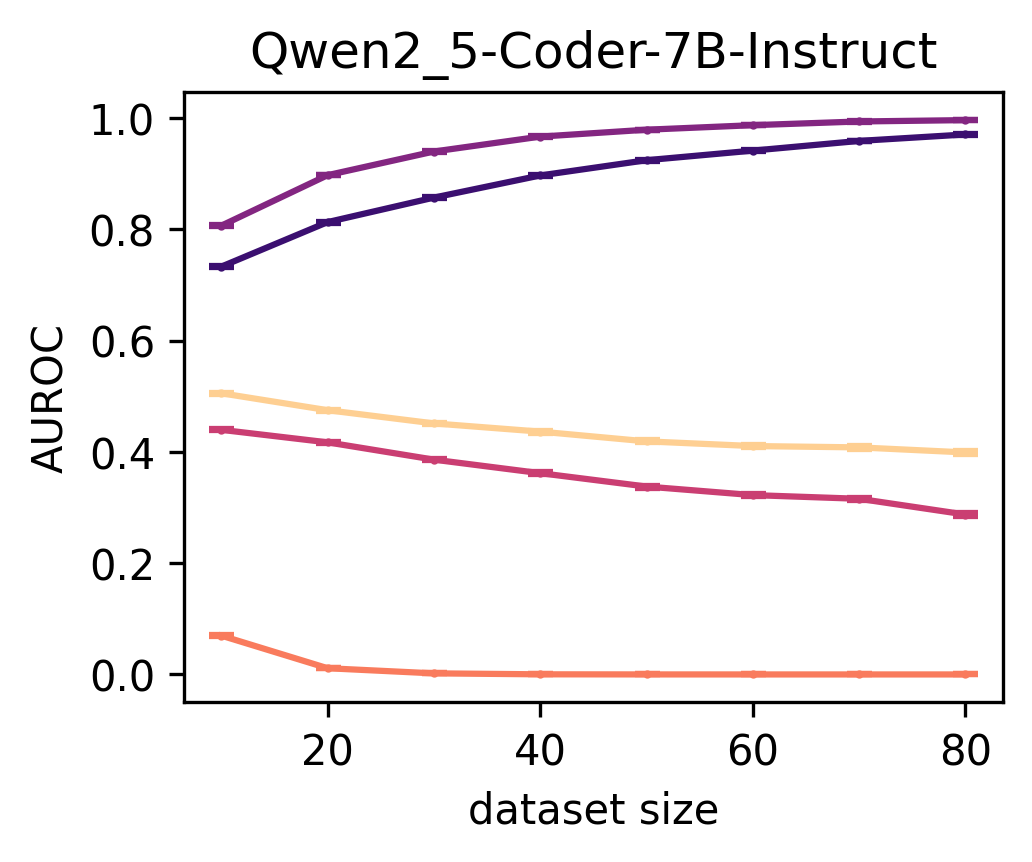} \\
\small Granite-3B-Code-Instruct-2K & \small DeepSeek-Coder-6.7B-Instruct & \small Qwen2.5-Coder-7B-Instruct \\
\end{tabular}
\caption{AUROC on AuthSec as a function of dataset size, for CodeSIFT ($L_1$, $L_2$) and the three static baselines jointly. Columns correspond to Granite-3B-Code-Instruct-2K, DeepSeek-Coder-6.7B-Instruct, and Qwen2.5-Coder-7B-Instruct.}
\label{fig:auth-standalone1}
\end{figure}

\begin{figure}[ht]
\centering
\begin{tabular}{ccc}
\includegraphics[width=0.30\linewidth]{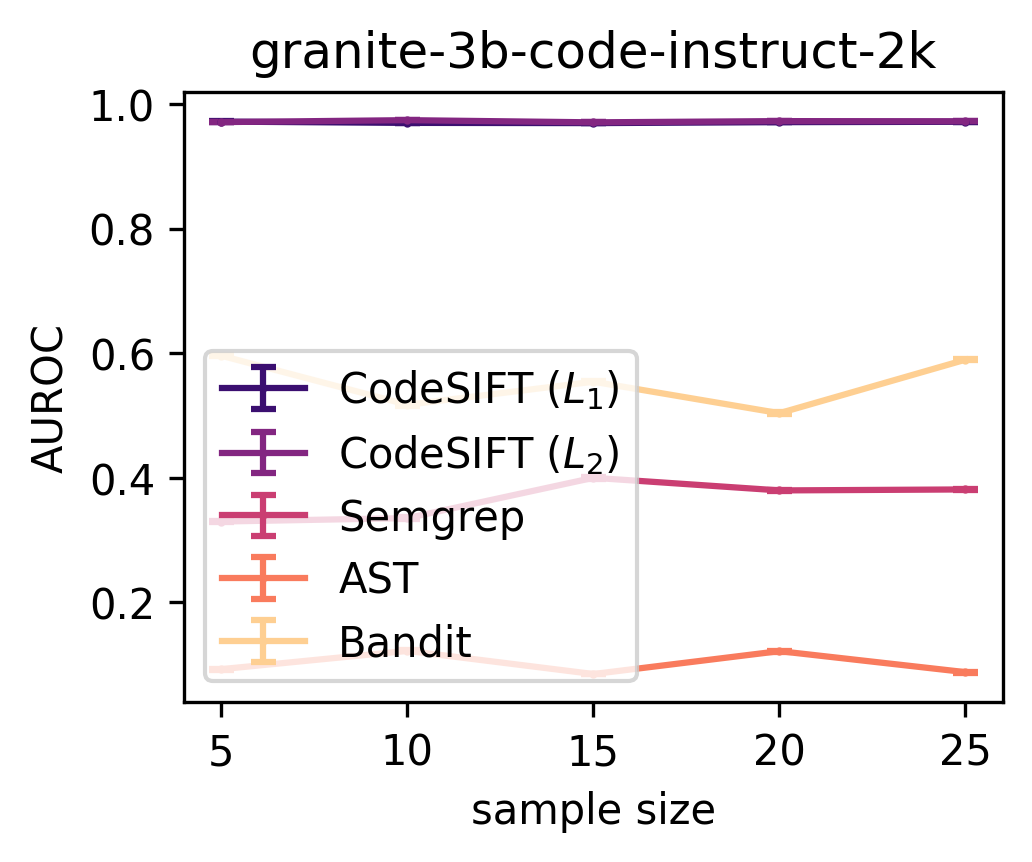} & \includegraphics[width=0.30\linewidth]{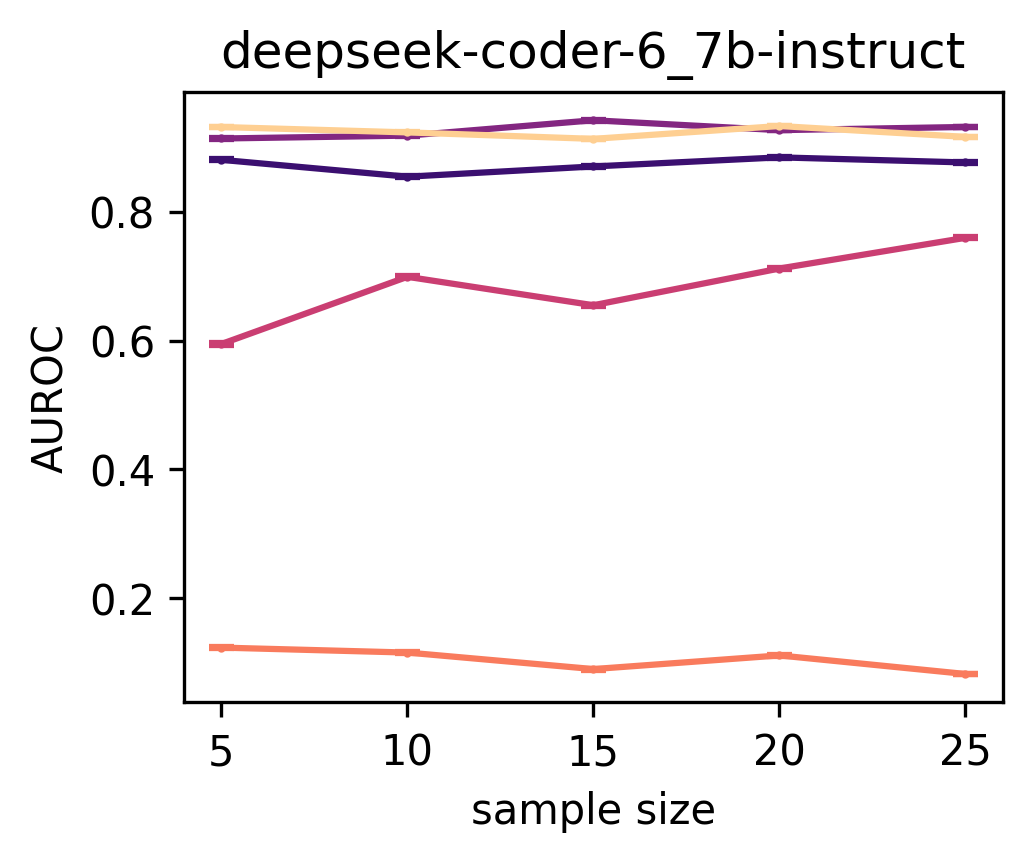} & \includegraphics[width=0.30\linewidth]{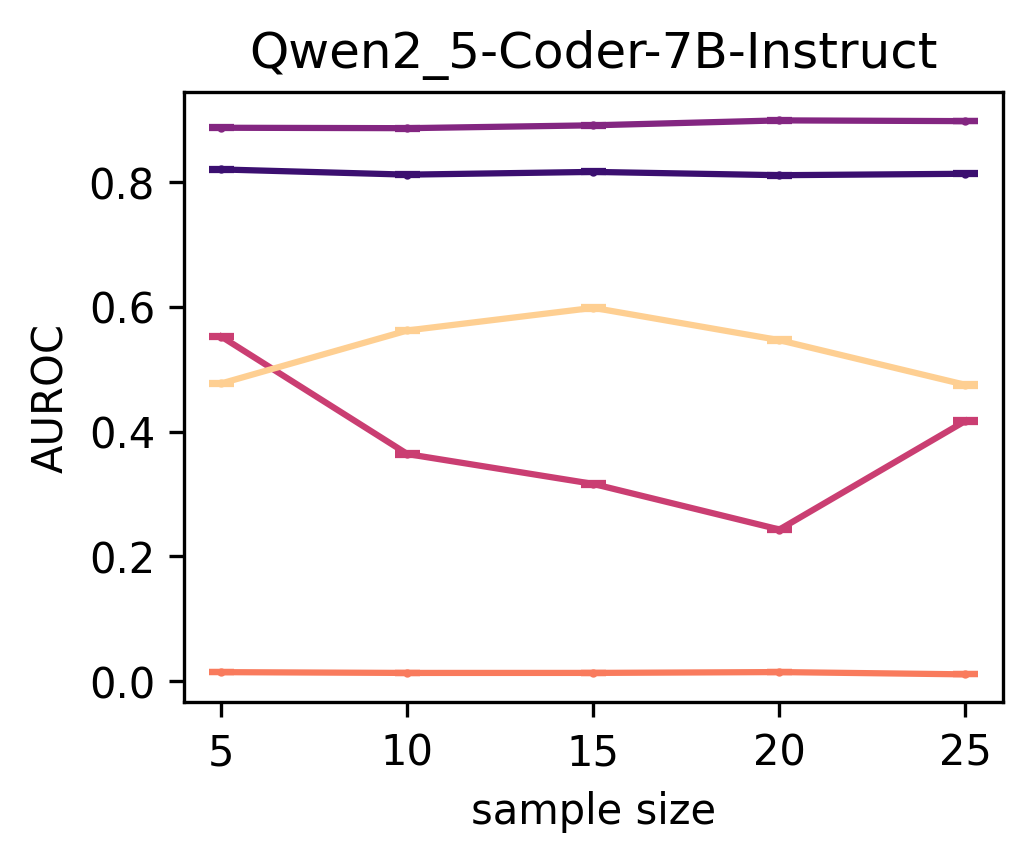} \\
\small Granite-3B-Code-Instruct-2K & \small DeepSeek-Coder-6.7B-Instruct & \small Qwen2.5-Coder-7B-Instruct \\
\end{tabular}
\caption{AUROC on AuthSec as a function of number of sampled completions, for CodeSIFT ($L_1$, $L_2$) and the three static baselines jointly. Columns correspond to Granite-3B-Code-Instruct-2K, DeepSeek-Coder-6.7B-Instruct, and Qwen2.5-Coder-7B-Instruct.}
\label{fig:auth-standalone2}
\end{figure}

\begin{figure}[ht]
\centering
\begin{tabular}{ccc}
\includegraphics[width=0.30\linewidth]{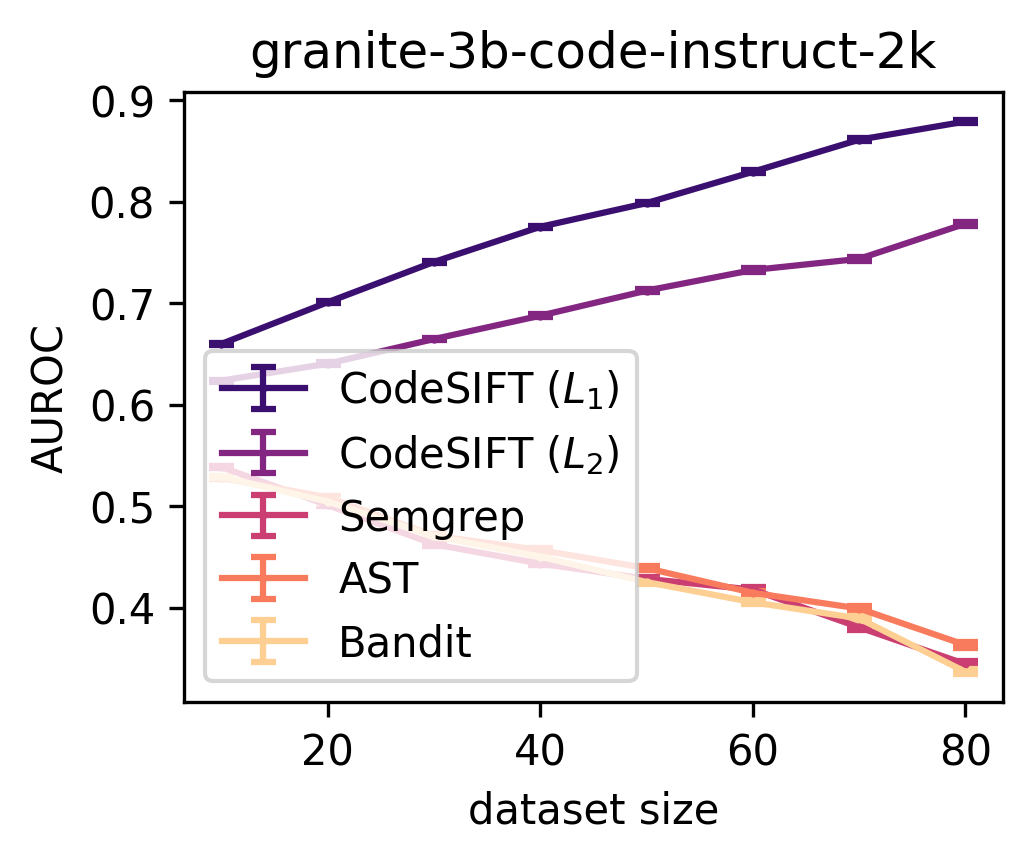} & \includegraphics[width=0.30\linewidth]{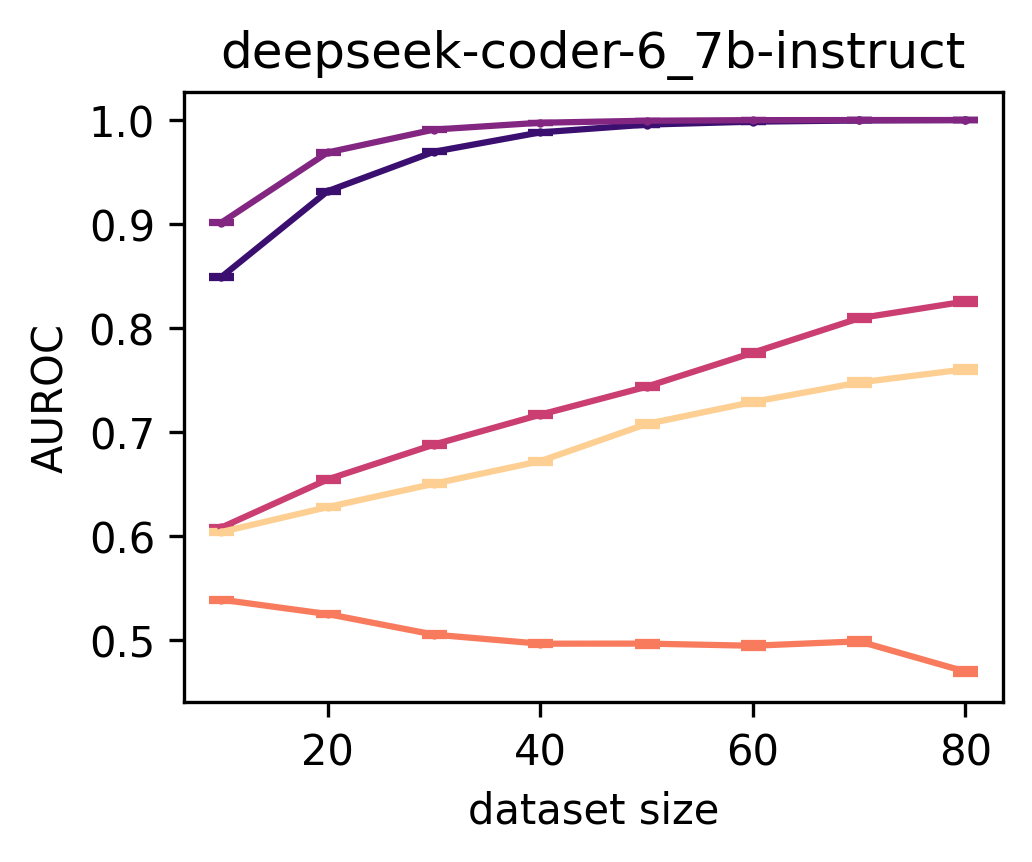} & \includegraphics[width=0.30\linewidth]{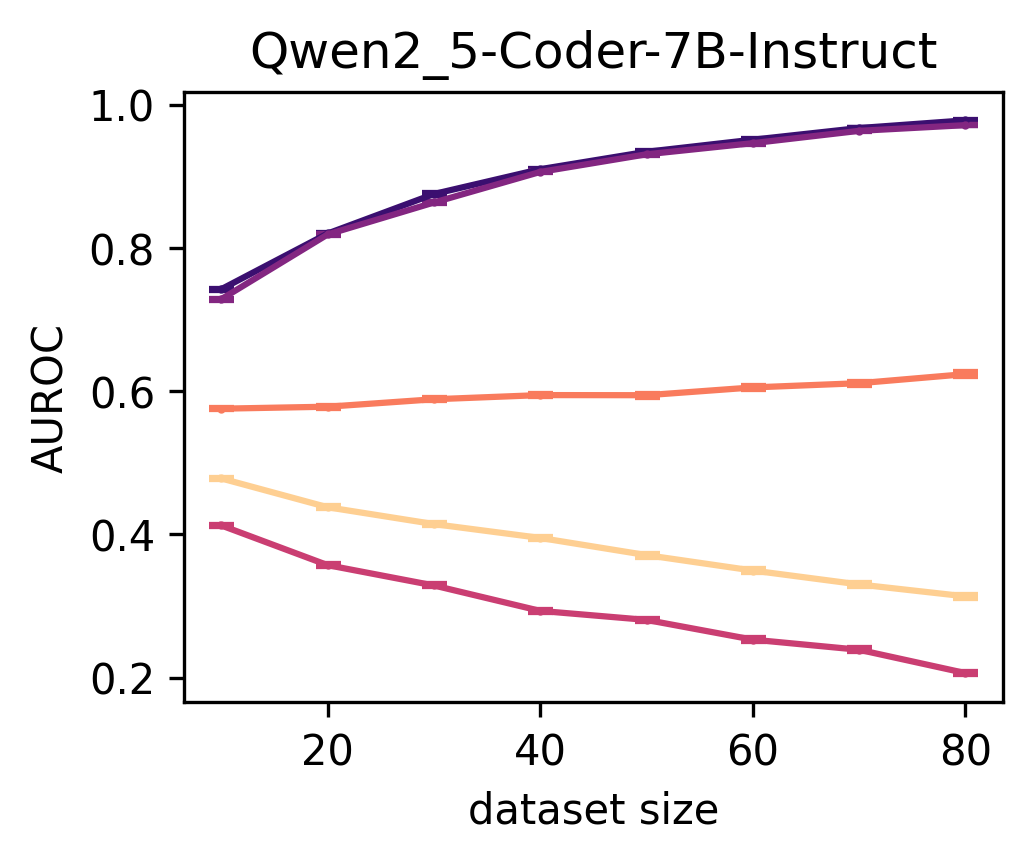} \\
\small Granite-3B-Code-Instruct-2K & \small DeepSeek-Coder-6.7B-Instruct & \small Qwen2.5-Coder-7B-Instruct \\
\end{tabular}
\caption{AUROC on InfraCloud as a function of dataset size, for CodeSIFT ($L_1$, $L_2$) and the three static baselines jointly. Columns correspond to Granite-3B-Code-Instruct-2K, DeepSeek-Coder-6.7B-Instruct, and Qwen2.5-Coder-7B-Instruct.}
\label{fig:infrastructure-standalone1}
\end{figure}

\begin{figure}[ht]
\centering
\begin{tabular}{ccc}
\includegraphics[width=0.30\linewidth]{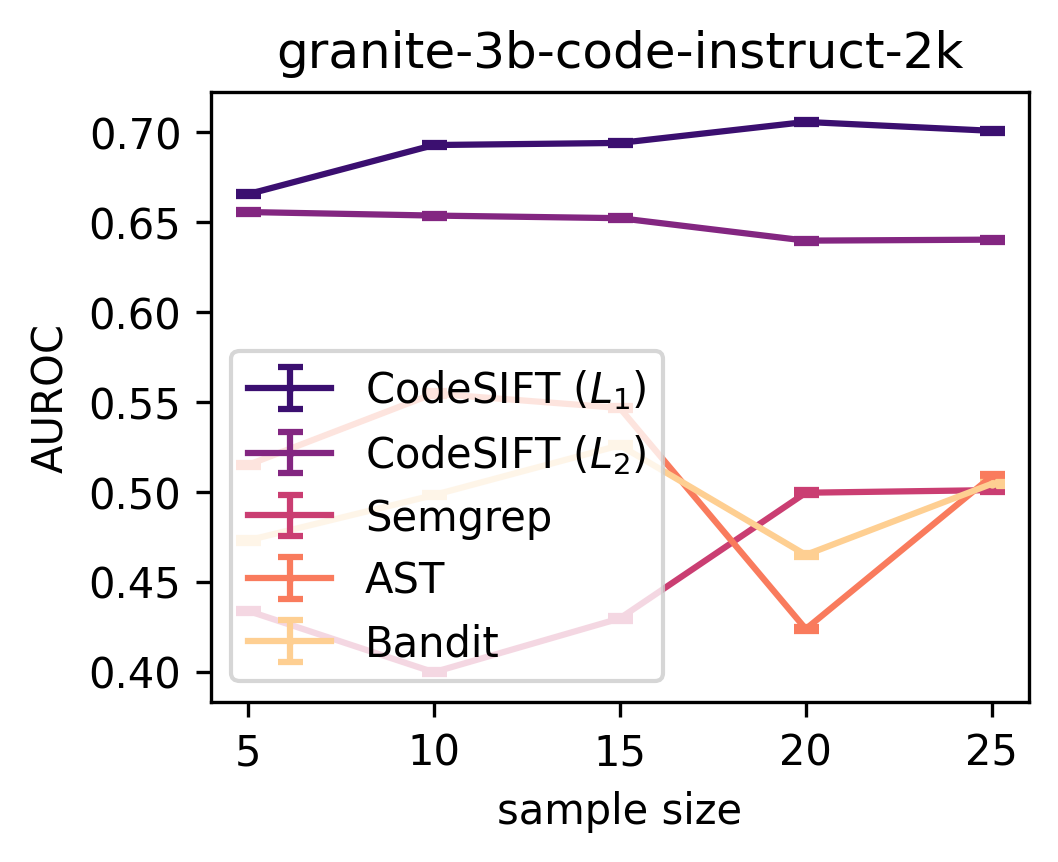} & \includegraphics[width=0.30\linewidth]{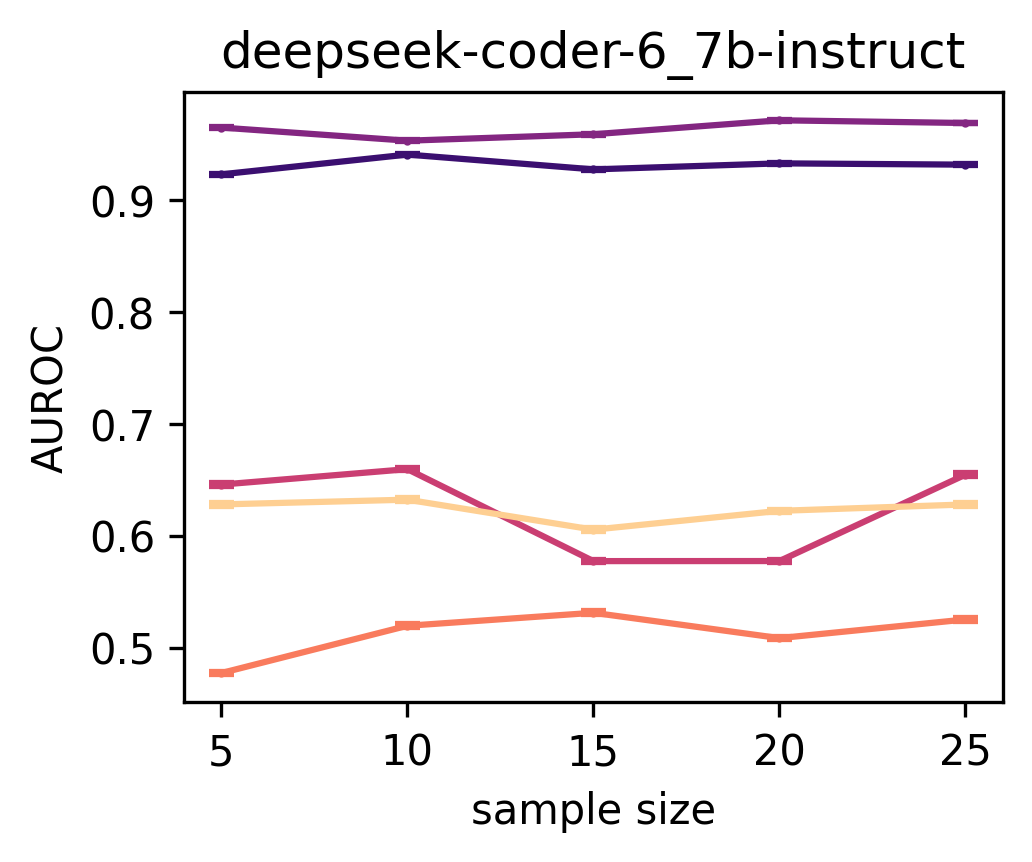} & \includegraphics[width=0.30\linewidth]{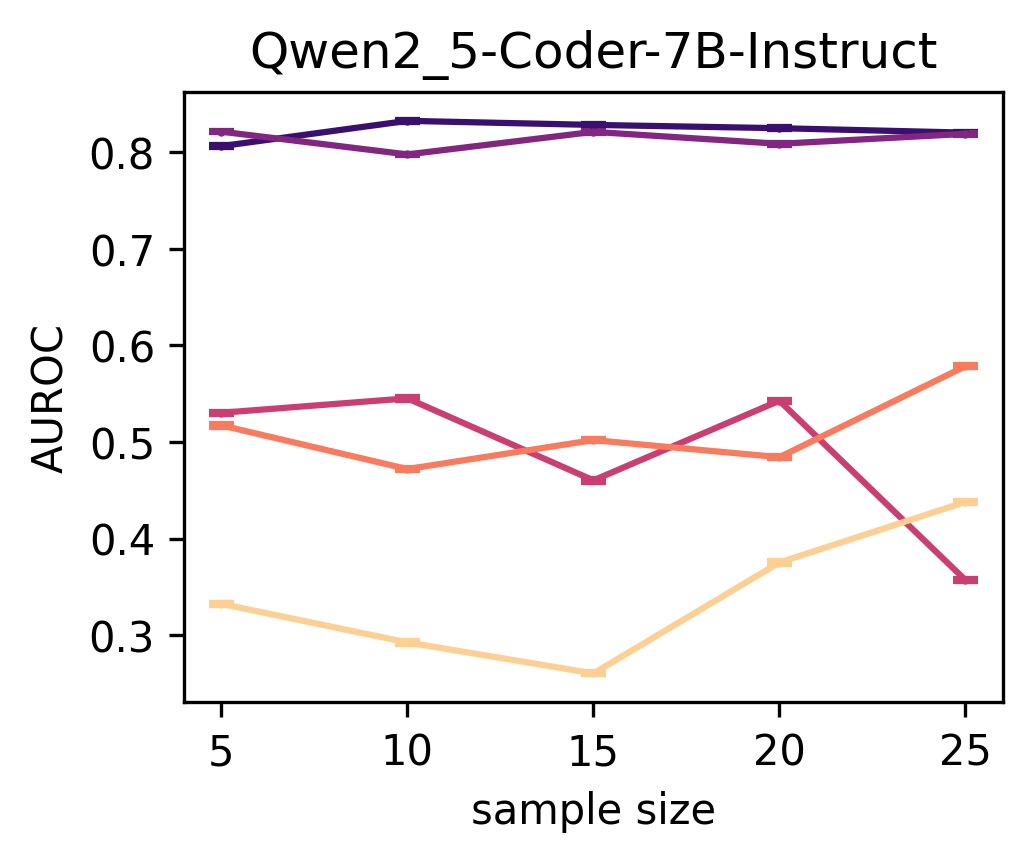} \\
\small Granite-3B-Code-Instruct-2K & \small DeepSeek-Coder-6.7B-Instruct & \small Qwen2.5-Coder-7B-Instruct \\
\end{tabular}
\caption{AUROC on InfraCloud as a function of number of sampled completions, for CodeSIFT ($L_1$, $L_2$) and the three static baselines jointly. Columns correspond to Granite-3B-Code-Instruct-2K, DeepSeek-Coder-6.7B-Instruct, and Qwen2.5-Coder-7B-Instruct.}
\label{fig:infrastructure-standalone2}
\end{figure}

\FloatBarrier

\subsection{Dataset size versus injection rate}

Figures~\ref{fig:auth-granite-3b-code-instruct-2k-dataset_size_injection_rate}--\ref{fig:auth-Qwen2_5-Coder-7B-Instruct-dataset_size_injection_rate} report AUROC jointly as a function of dataset size and injection rate on AuthSec, and Figures~\ref{fig:infrastructure-granite-3b-code-instruct-2k-dataset_size_injection_rate}--\ref{fig:infrastructure-Qwen2_5-Coder-7B-Instruct-dataset_size_injection_rate} report the same on InfraCloud, one figure per model. CodeSIFT's accuracy increases along both axes, larger candidate sets and higher contamination rates each make detection easier, and the $L_1$ and $L_2$ variants behave near-identically throughout. Bandit stays largely flat around 0.5 to 0.6 regardless of either quantity, indicating no real dependence on contamination level or candidate set size, while Semgrep and the AST Taint Tracker decline as either axis grows, in several settings falling to an AUROC near zero at the largest dataset sizes and injection rates, that is, becoming more confidently wrong as more evidence accumulates rather than merely uninformative.

\begin{figure}[ht]
\centering
\includegraphics[width=\linewidth]{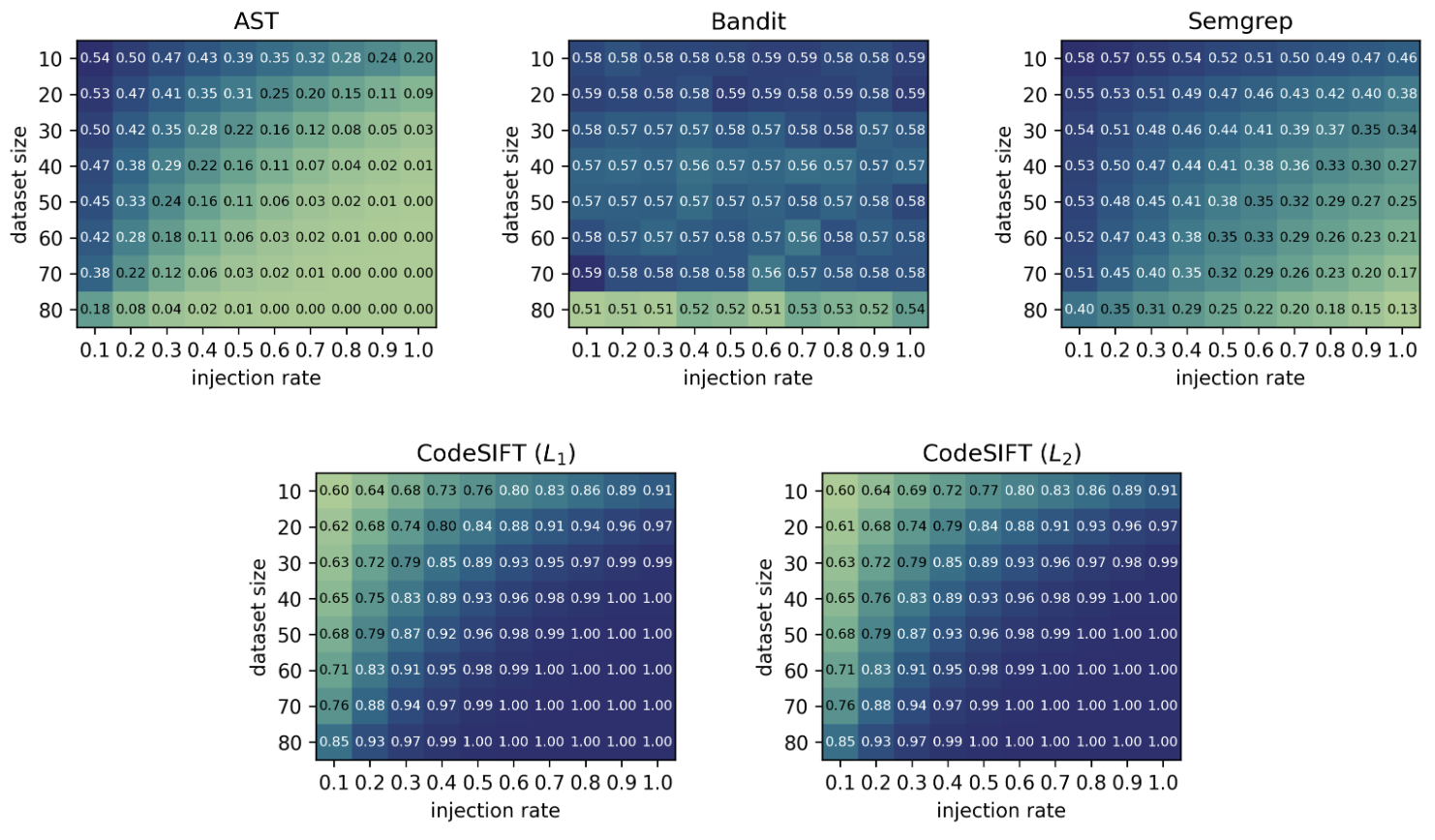}
\caption{AUROC on AuthSec for Granite-3B-Code-Instruct-2K, as a joint function of dataset size and injection rate, for the three static baselines (top row) and CodeSIFT ($L_1$, $L_2$, bottom row). CodeSIFT's AUROC increases monotonically along both axes, while the baselines remain largely flat or decline as either quantity grows.}
\label{fig:auth-granite-3b-code-instruct-2k-dataset_size_injection_rate}
\end{figure}

\begin{figure}[ht]
\centering
\includegraphics[width=\linewidth]{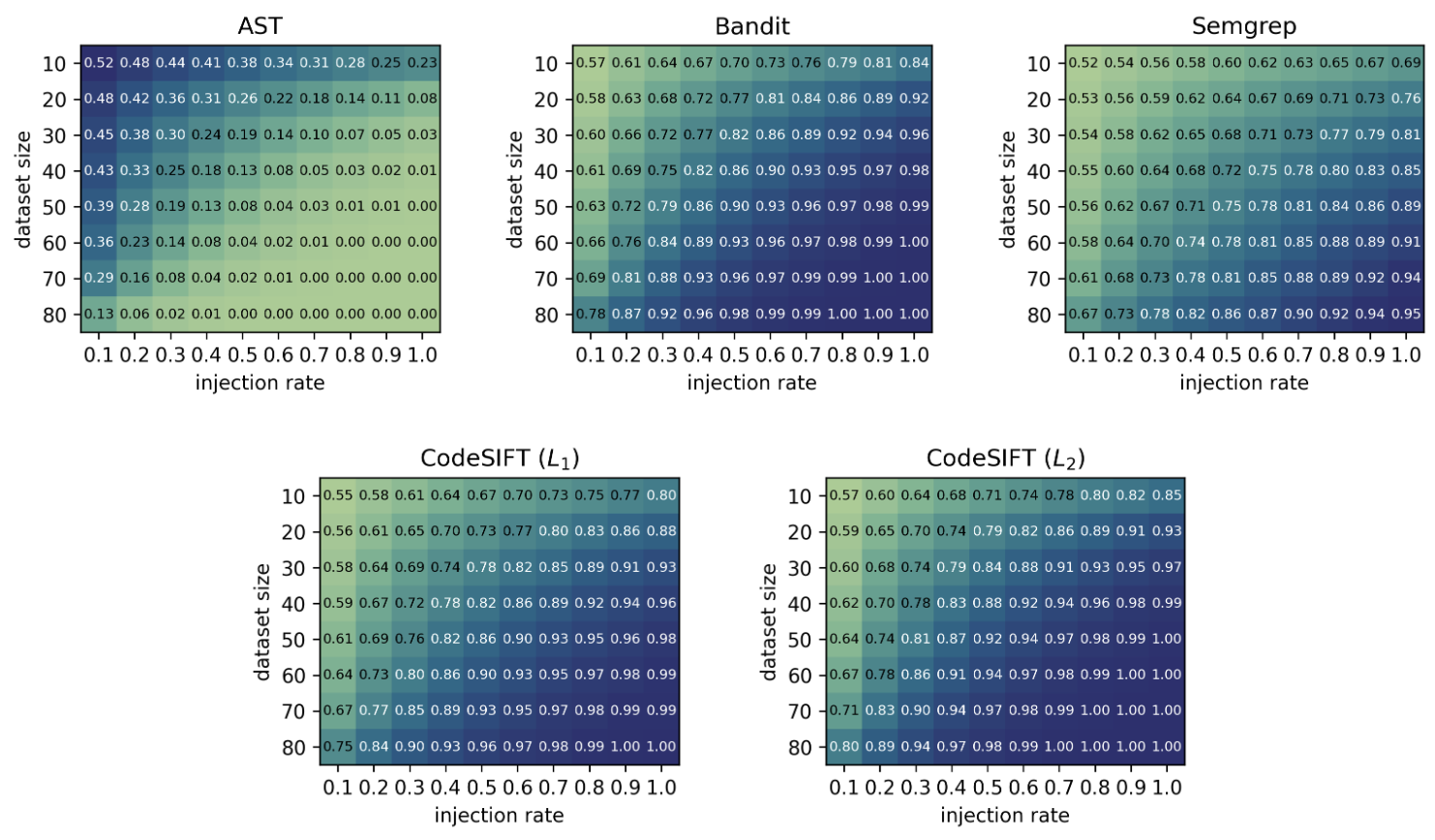}
\caption{AUROC on AuthSec for DeepSeek-Coder-6.7B-Instruct, as a joint function of dataset size and injection rate, for the three static baselines (top row) and CodeSIFT ($L_1$, $L_2$, bottom row). CodeSIFT's AUROC increases monotonically along both axes, while the baselines remain largely flat or decline as either quantity grows.}
\label{fig:auth-deepseek-coder-6_7b-instruct-dataset_size_injection_rate}
\end{figure}

\begin{figure}[ht]
\centering
\includegraphics[width=\linewidth]{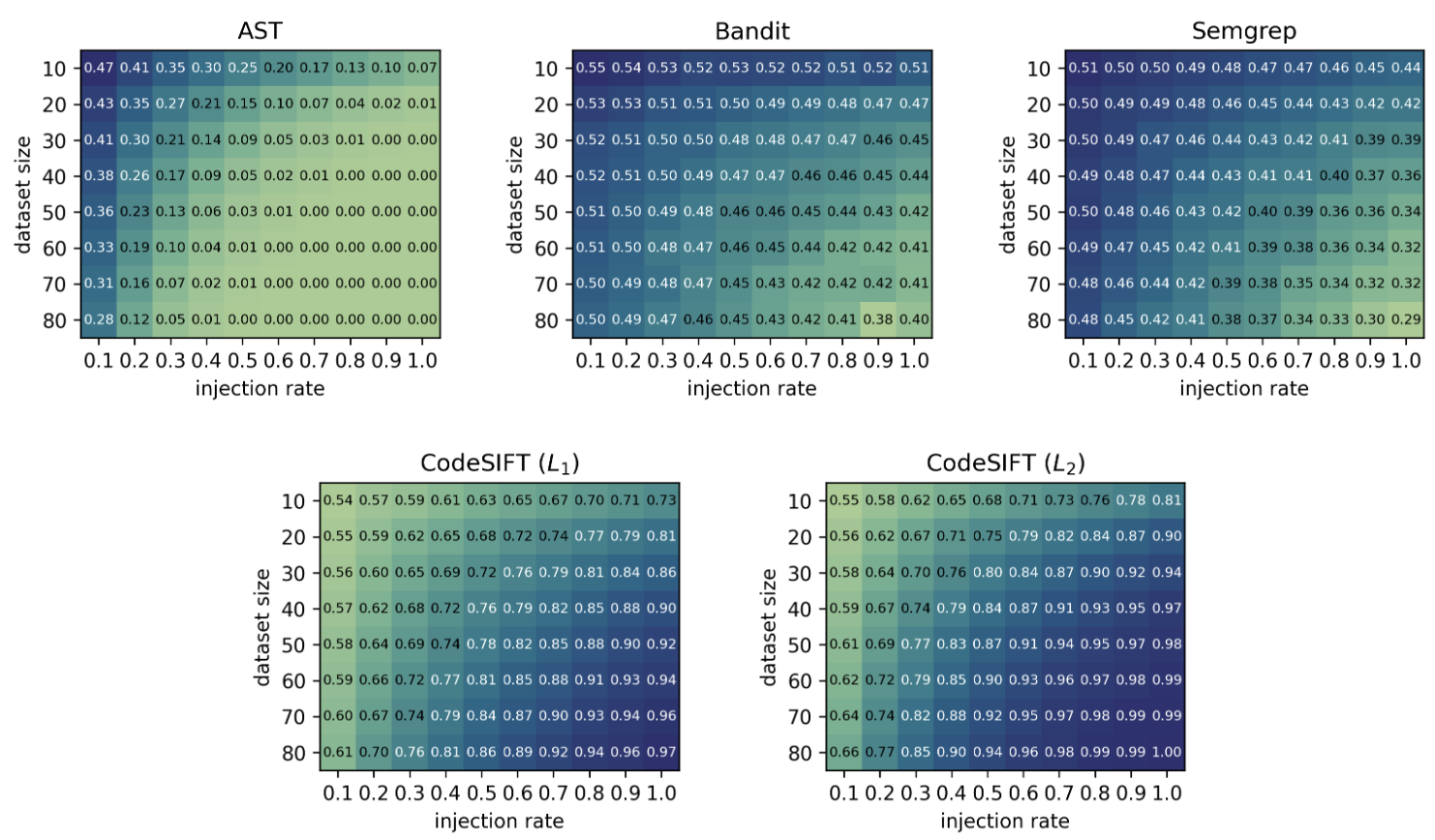}
\caption{AUROC on AuthSec for Qwen2.5-Coder-7B-Instruct, as a joint function of dataset size and injection rate, for the three static baselines (top row) and CodeSIFT ($L_1$, $L_2$, bottom row). CodeSIFT's AUROC increases monotonically along both axes, while the baselines remain largely flat or decline as either quantity grows.}
\label{fig:auth-Qwen2_5-Coder-7B-Instruct-dataset_size_injection_rate}
\end{figure}

\begin{figure}[ht]
\centering
\includegraphics[width=\linewidth]{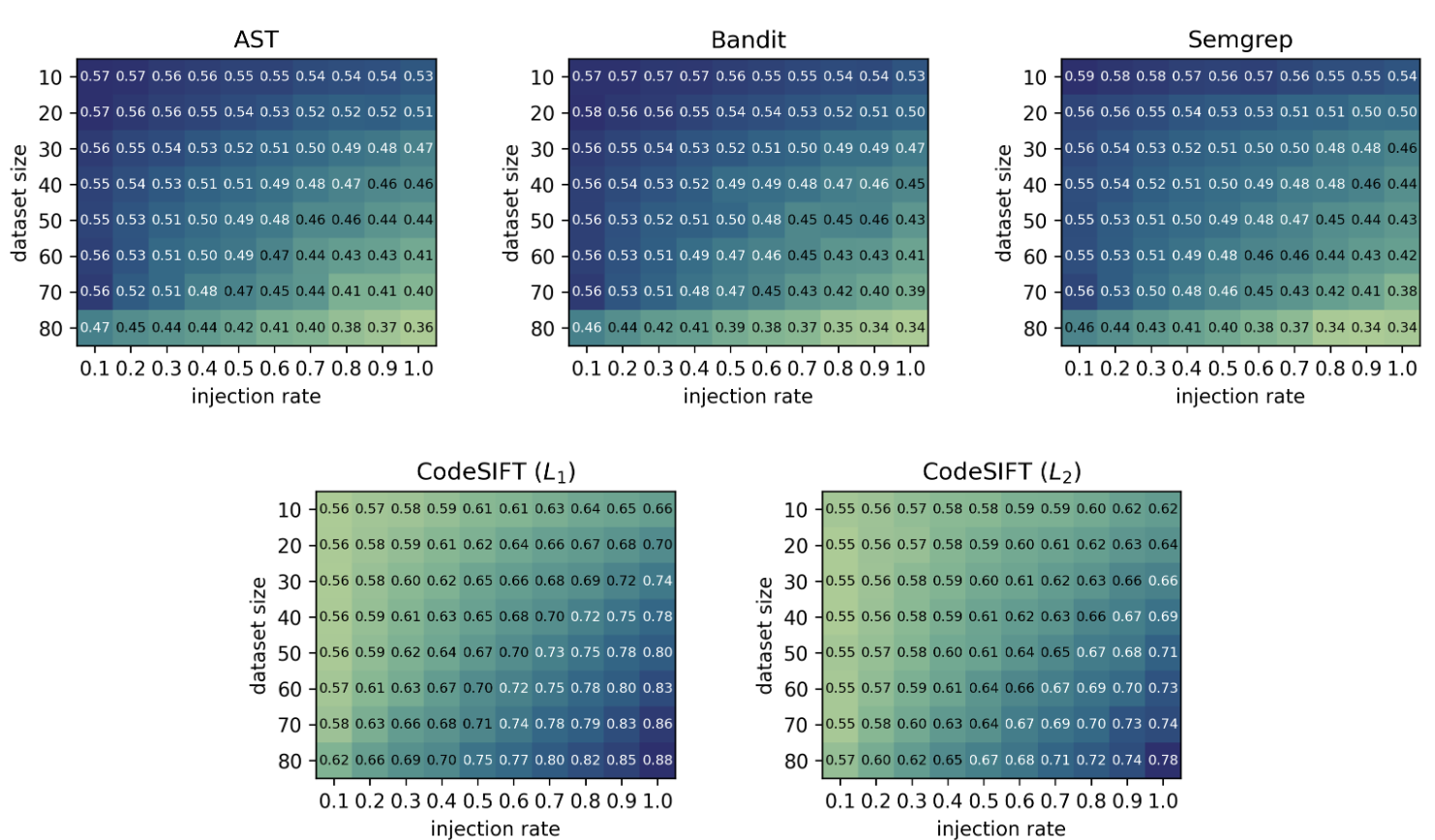}
\caption{AUROC on InfraCloud for Granite-3B-Code-Instruct-2K, as a joint function of dataset size and injection rate, for the three static baselines (top row) and CodeSIFT ($L_1$, $L_2$, bottom row). CodeSIFT's AUROC increases monotonically along both axes, while the baselines remain largely flat or decline as either quantity grows.}
\label{fig:infrastructure-granite-3b-code-instruct-2k-dataset_size_injection_rate}
\end{figure}

\begin{figure}[ht]
\centering
\includegraphics[width=\linewidth]{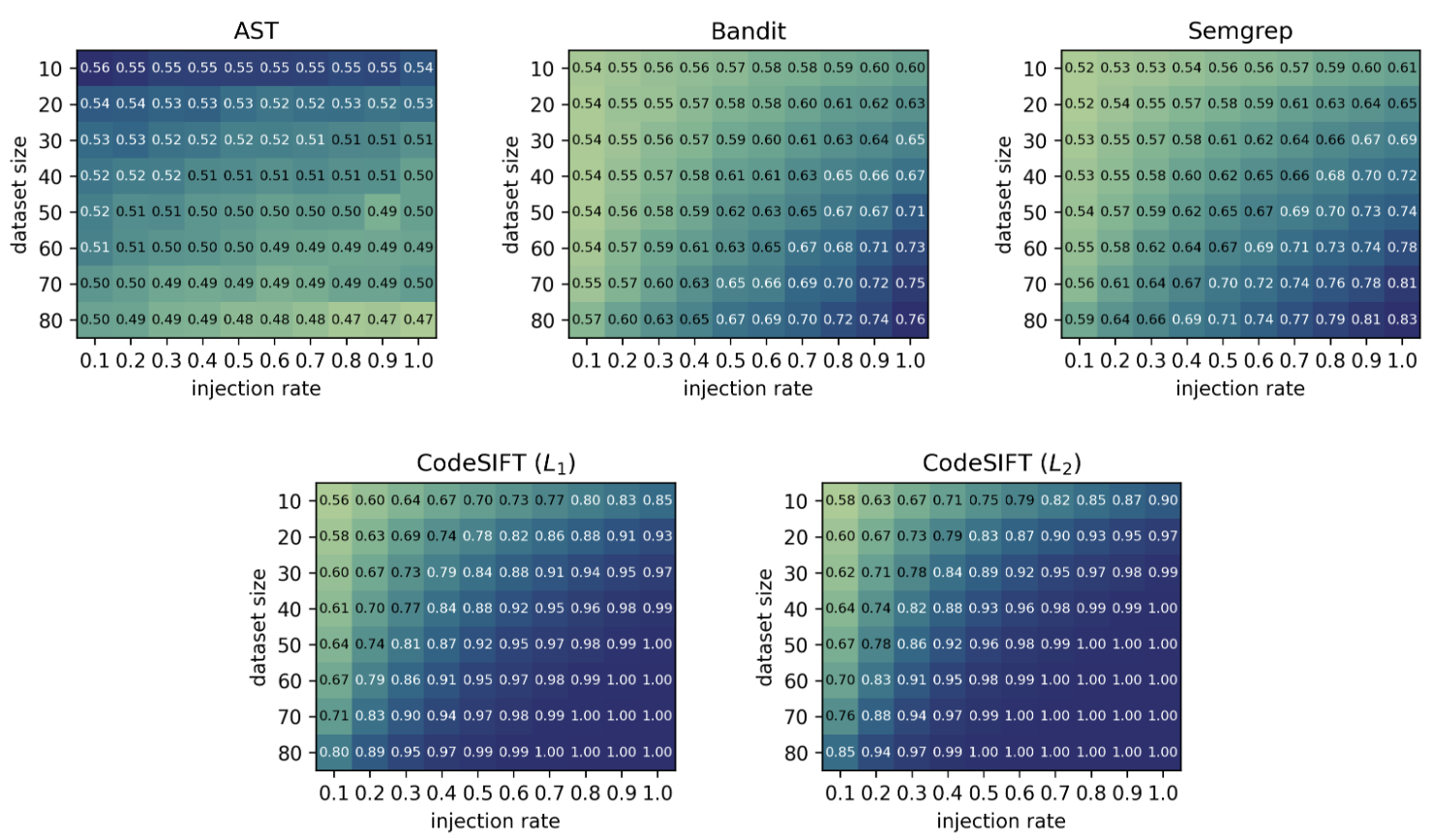}
\caption{AUROC on InfraCloud for DeepSeek-Coder-6.7B-Instruct, as a joint function of dataset size and injection rate, for the three static baselines (top row) and CodeSIFT ($L_1$, $L_2$, bottom row). CodeSIFT's AUROC increases monotonically along both axes, while the baselines remain largely flat or decline as either quantity grows.}
\label{fig:infrastructure-deepseek-coder-6_7b-instruct-dataset_size_injection_rate}
\end{figure}

\begin{figure}[ht]
\centering
\includegraphics[width=\linewidth]{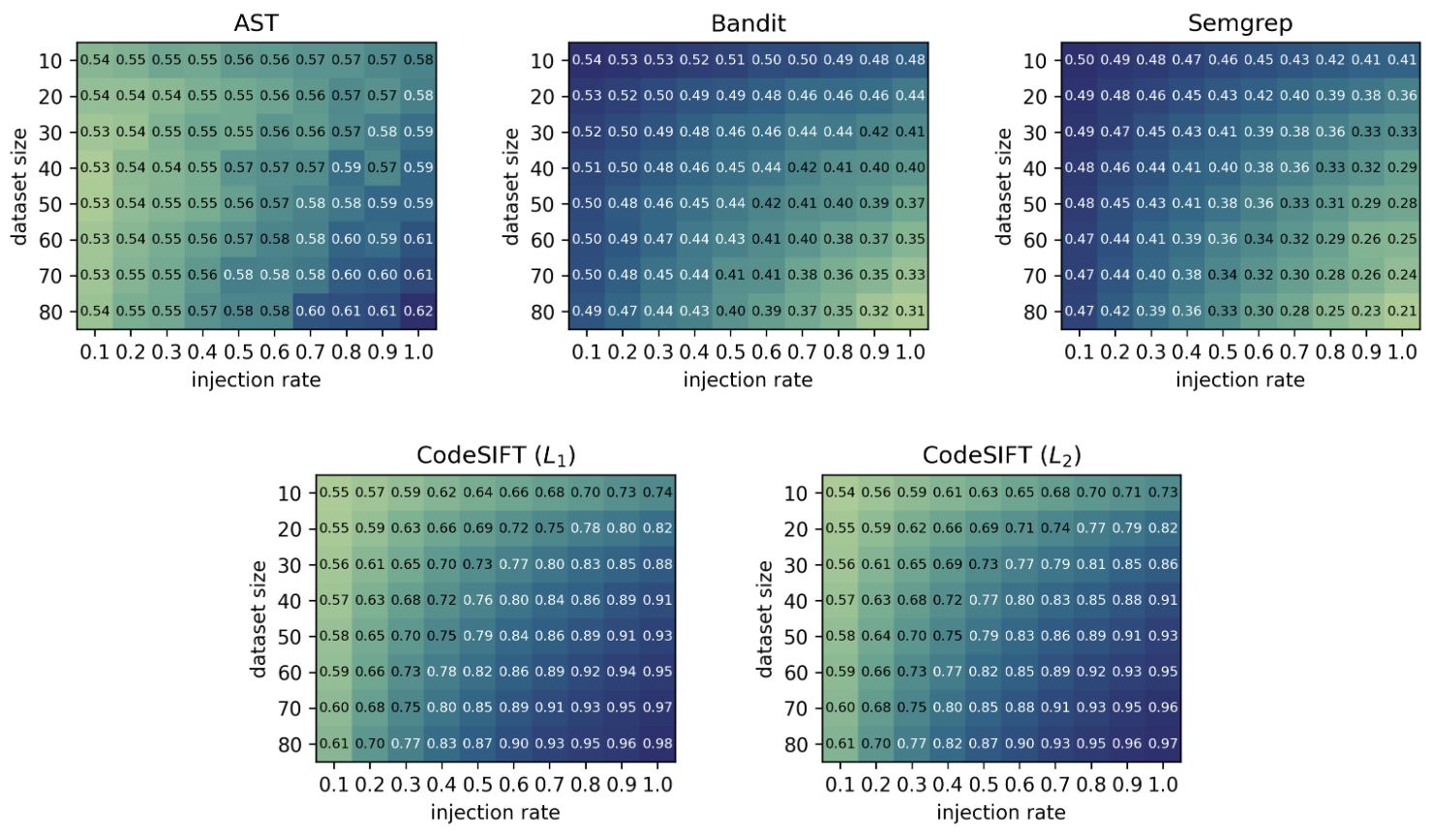}
\caption{AUROC on InfraCloud for Qwen2.5-Coder-7B-Instruct, as a joint function of dataset size and injection rate, for the three static baselines (top row) and CodeSIFT ($L_1$, $L_2$, bottom row). CodeSIFT's AUROC increases monotonically along both axes, while the baselines remain largely flat or decline as either quantity grows.}
\label{fig:infrastructure-Qwen2_5-Coder-7B-Instruct-dataset_size_injection_rate}
\end{figure}

\FloatBarrier

\subsection{Dataset size versus number of sampled completions}

Figures~\ref{fig:auth-granite-3b-code-instruct-2k-dataset_size_n_samples}--\ref{fig:auth-Qwen2_5-Coder-7B-Instruct-dataset_size_n_samples} report AUROC jointly as a function of dataset size and the number of sampled completions per prompt on AuthSec, and Figures~\ref{fig:infrastructure-granite-3b-code-instruct-2k-dataset_size_n_samples}--\ref{fig:infrastructure-Qwen2_5-Coder-7B-Instruct-dataset_size_n_samples} report the same on InfraCloud, one figure per model, at a fixed injection rate. CodeSIFT's accuracy tracks dataset size almost exclusively, increasing steadily as the candidate set grows, and is essentially flat across the number of sampled completions, so the detector does not need many completions per prompt to work. The baselines show no clear structure along either axis, their AUROC fluctuates within a narrow band around chance regardless of how much data or how many completions are available.

\begin{figure}[ht]
\centering
\includegraphics[width=\linewidth]{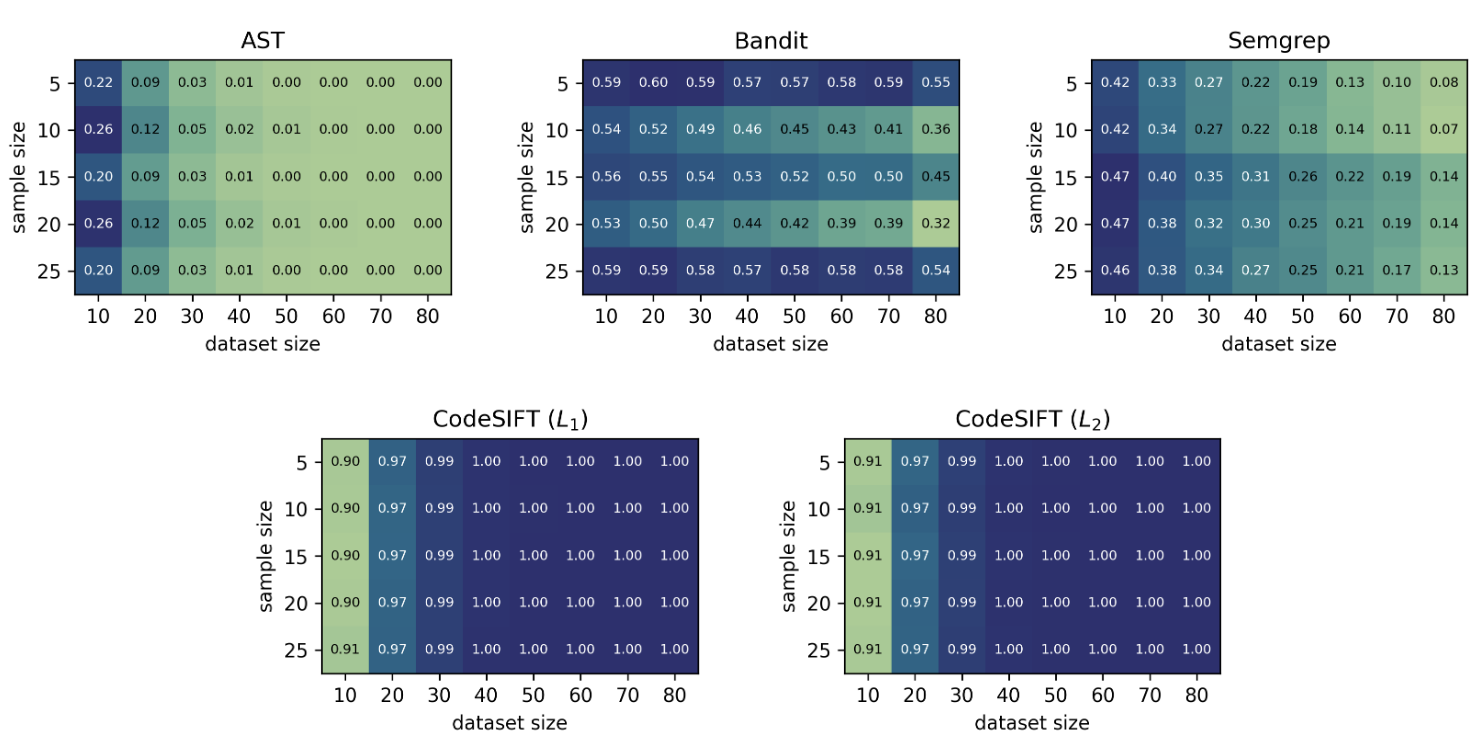}
\caption{AUROC on AuthSec for Granite-3B-Code-Instruct-2K, as a joint function of dataset size and number of sampled completions per prompt, for the three static baselines (top row) and CodeSIFT ($L_1$, $L_2$, bottom row), at a fixed injection rate. CodeSIFT's AUROC depends almost entirely on dataset size and is largely insensitive to the number of sampled completions, indicating the test is sample efficient. The baselines show no clear dependence on either axis.}
\label{fig:auth-granite-3b-code-instruct-2k-dataset_size_n_samples}
\end{figure}

\begin{figure}[ht]
\centering
\includegraphics[width=\linewidth]{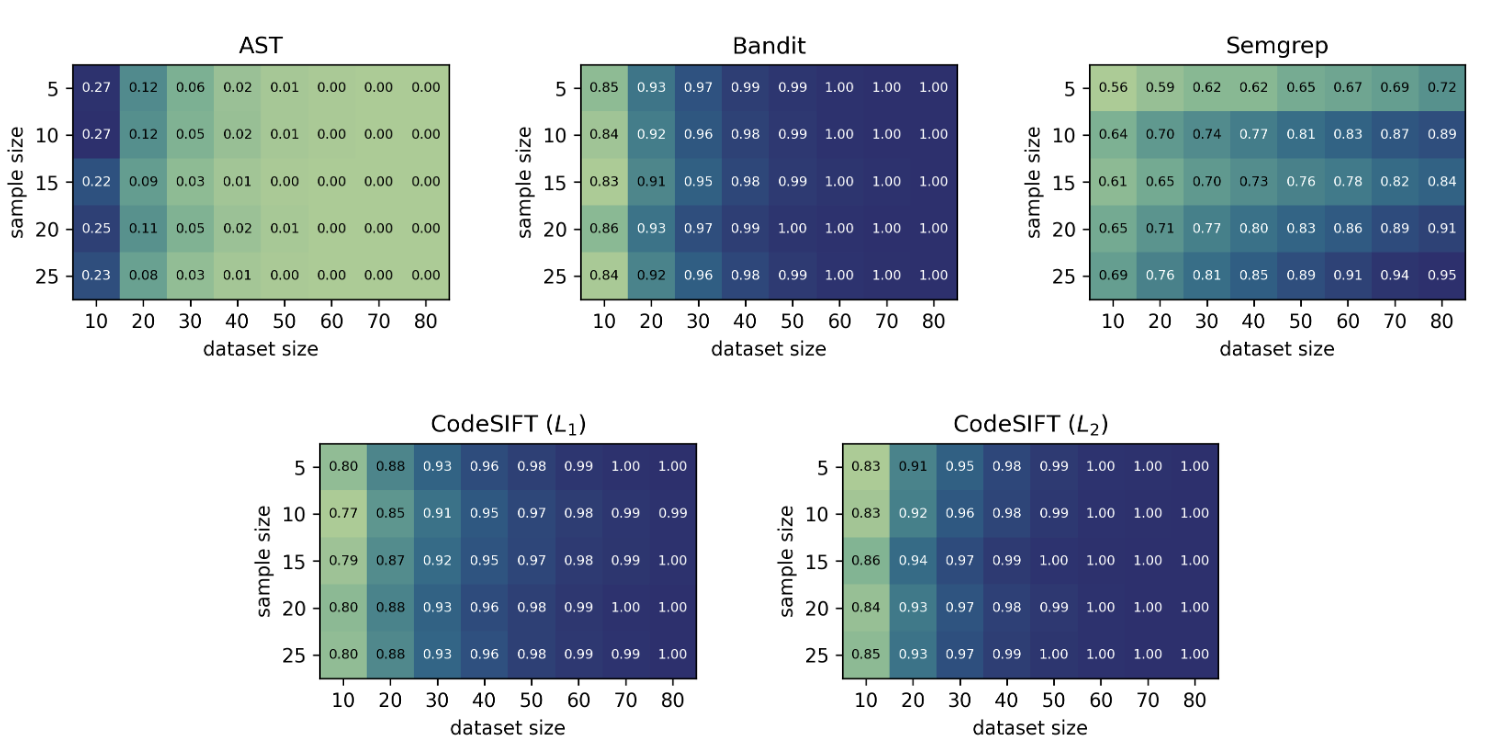}
\caption{AUROC on AuthSec for DeepSeek-Coder-6.7B-Instruct, as a joint function of dataset size and number of sampled completions per prompt, for the three static baselines (top row) and CodeSIFT ($L_1$, $L_2$, bottom row), at a fixed injection rate. CodeSIFT's AUROC depends almost entirely on dataset size and is largely insensitive to the number of sampled completions, indicating the test is sample efficient. The baselines show no clear dependence on either axis.}
\label{fig:auth-deepseek-coder-6_7b-instruct-dataset_size_n_samples}
\end{figure}

\begin{figure}[ht]
\centering
\includegraphics[width=\linewidth]{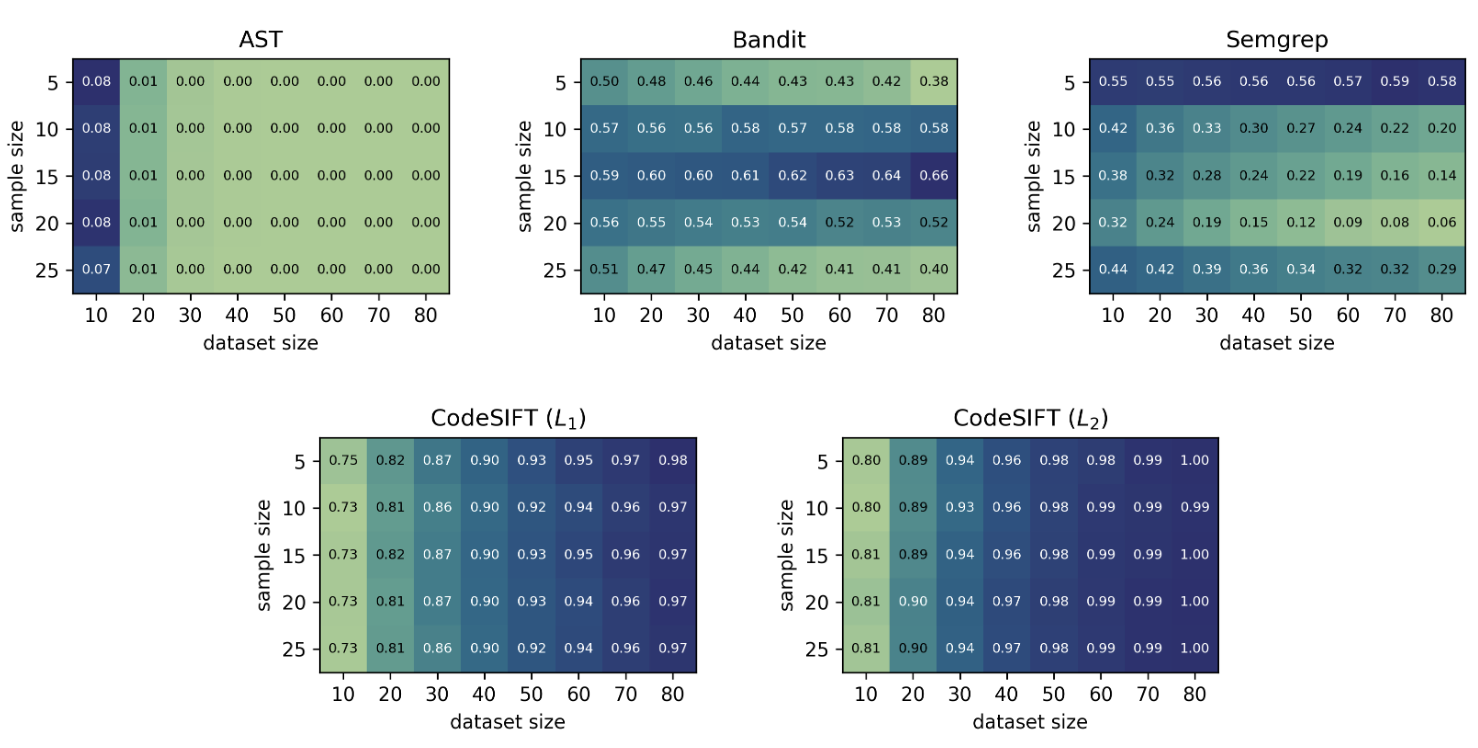}
\caption{AUROC on AuthSec for Qwen2.5-Coder-7B-Instruct, as a joint function of dataset size and number of sampled completions per prompt, for the three static baselines (top row) and CodeSIFT ($L_1$, $L_2$, bottom row), at a fixed injection rate. CodeSIFT's AUROC depends almost entirely on dataset size and is largely insensitive to the number of sampled completions, indicating the test is sample efficient. The baselines show no clear dependence on either axis.}
\label{fig:auth-Qwen2_5-Coder-7B-Instruct-dataset_size_n_samples}
\end{figure}

\begin{figure}[ht]
\centering
\includegraphics[width=\linewidth]{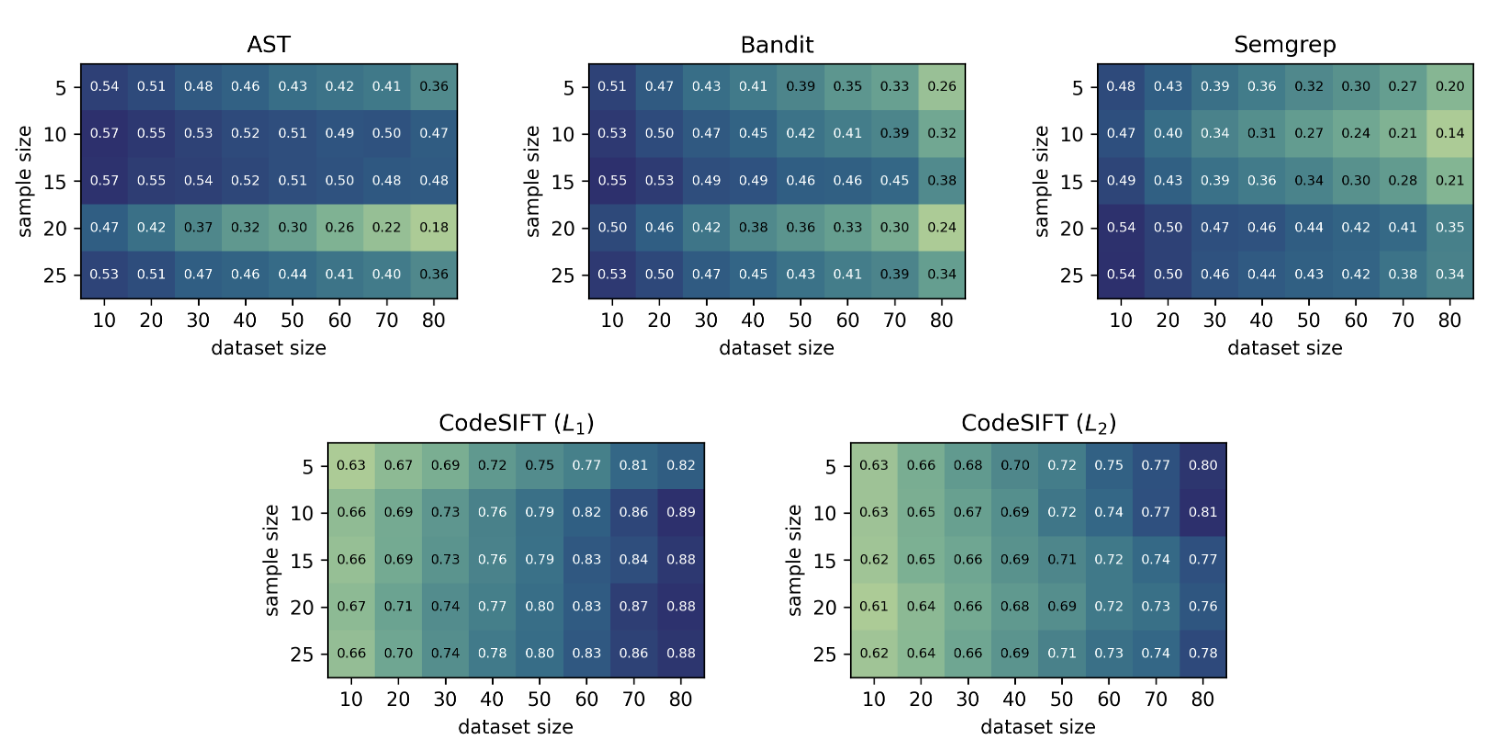}
\caption{AUROC on InfraCloud for Granite-3B-Code-Instruct-2K, as a joint function of dataset size and number of sampled completions per prompt, for the three static baselines (top row) and CodeSIFT ($L_1$, $L_2$, bottom row), at a fixed injection rate. CodeSIFT's AUROC depends almost entirely on dataset size and is largely insensitive to the number of sampled completions, indicating the test is sample efficient. The baselines show no clear dependence on either axis.}
\label{fig:infrastructure-granite-3b-code-instruct-2k-dataset_size_n_samples}
\end{figure}

\begin{figure}[ht]
\centering
\includegraphics[width=\linewidth]{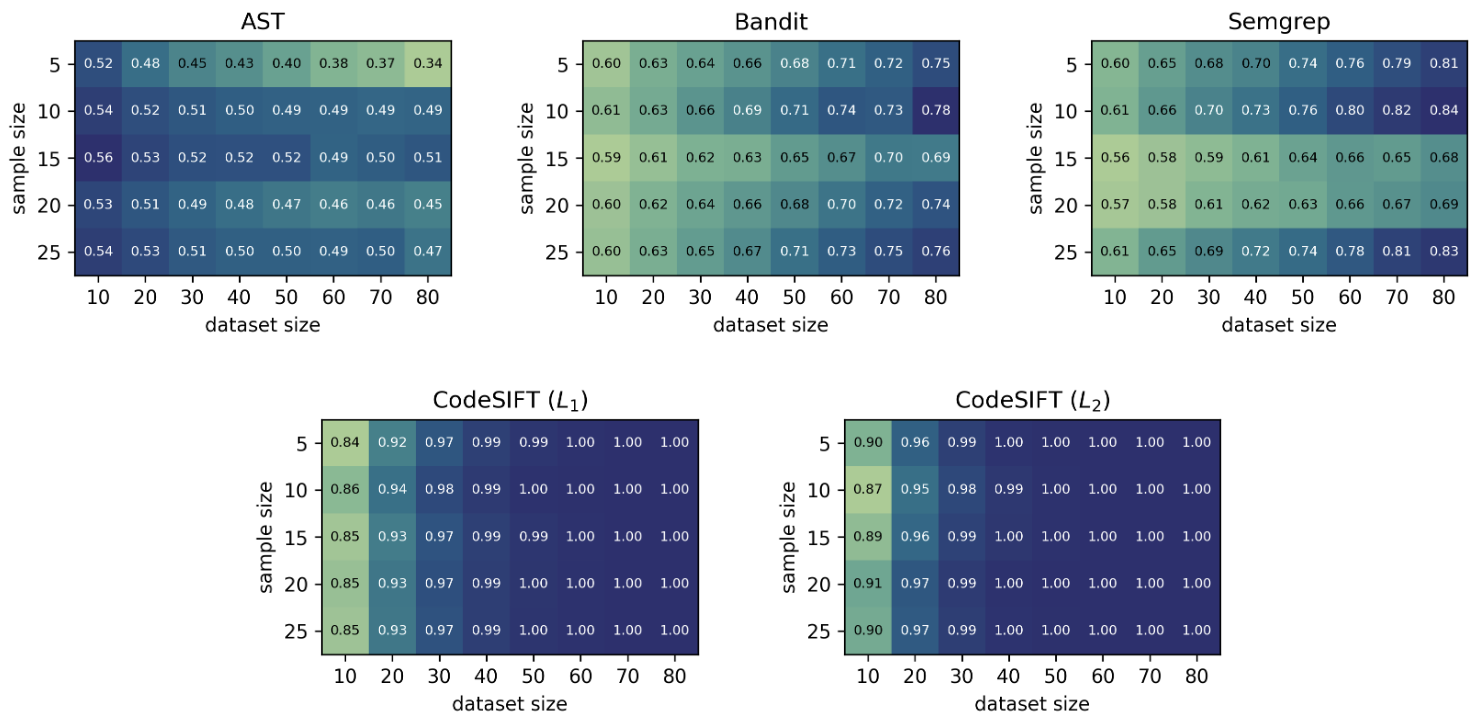}
\caption{AUROC on InfraCloud for DeepSeek-Coder-6.7B-Instruct, as a joint function of dataset size and number of sampled completions per prompt, for the three static baselines (top row) and CodeSIFT ($L_1$, $L_2$, bottom row), at a fixed injection rate. CodeSIFT's AUROC depends almost entirely on dataset size and is largely insensitive to the number of sampled completions, indicating the test is sample efficient. The baselines show no clear dependence on either axis.}
\label{fig:infrastructure-deepseek-coder-6_7b-instruct-dataset_size_n_samples}
\end{figure}

\begin{figure}[ht]
\centering
\includegraphics[width=\linewidth]{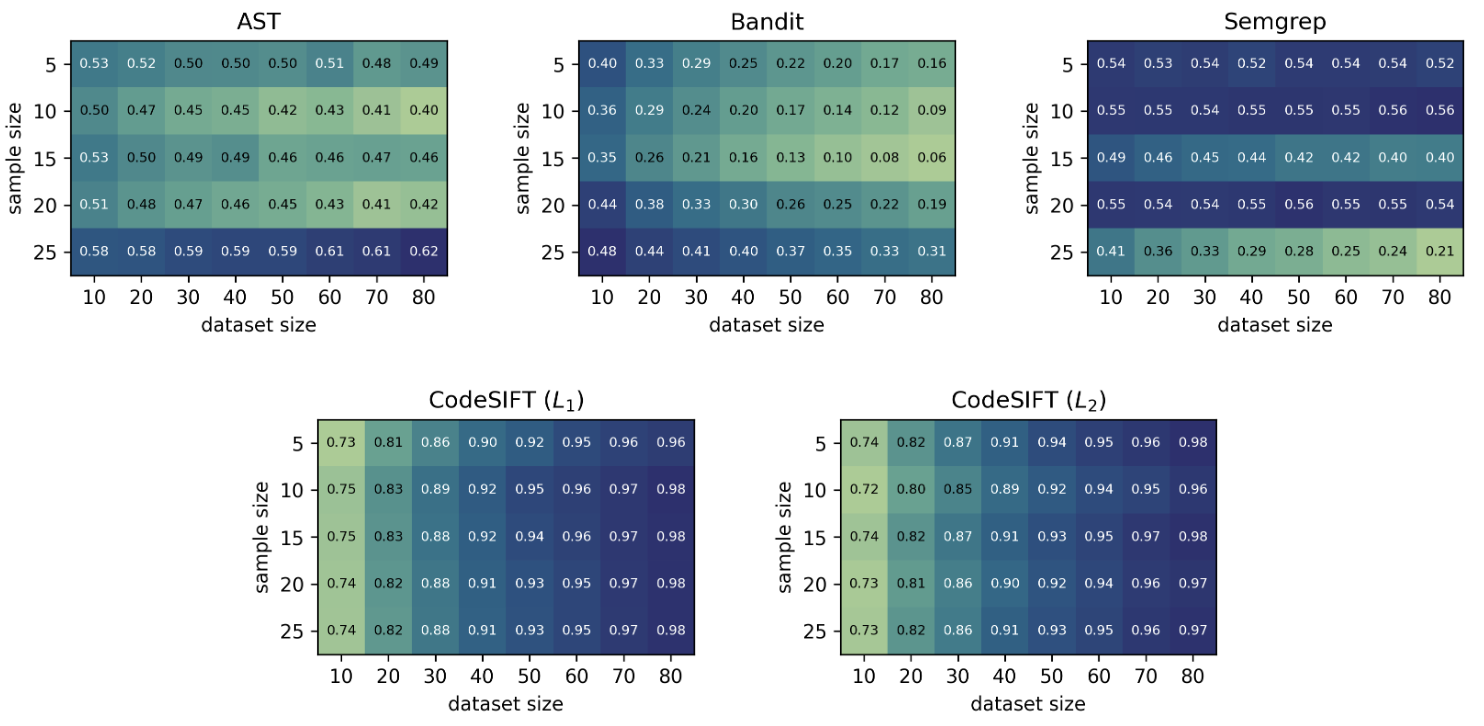}
\caption{AUROC on InfraCloud for Qwen2.5-Coder-7B-Instruct, as a joint function of dataset size and number of sampled completions per prompt, for the three static baselines (top row) and CodeSIFT ($L_1$, $L_2$, bottom row), at a fixed injection rate. CodeSIFT's AUROC depends almost entirely on dataset size and is largely insensitive to the number of sampled completions, indicating the test is sample efficient. The baselines show no clear dependence on either axis.}
\label{fig:infrastructure-Qwen2_5-Coder-7B-Instruct-dataset_size_n_samples}
\end{figure}

\FloatBarrier

\subsection{Injection rate versus number of sampled completions}

Figures~\ref{fig:auth-granite-3b-code-instruct-2k-injection_rate_n_samples}--\ref{fig:auth-Qwen2_5-Coder-7B-Instruct-injection_rate_n_samples} report AUROC jointly as a function of injection rate and the number of sampled completions per prompt on AuthSec, and Figures~\ref{fig:infrastructure-granite-3b-code-instruct-2k-injection_rate_n_samples}--\ref{fig:infrastructure-Qwen2_5-Coder-7B-Instruct-injection_rate_n_samples} report the same on InfraCloud, one figure per model, at a fixed dataset size. As in the previous sweep, CodeSIFT's accuracy tracks one axis almost exclusively, here the injection rate, increasing steadily as contamination grows, while remaining essentially flat across the number of sampled completions. The baselines again show no clear structure along either axis.

\begin{figure}[ht]
\centering
\includegraphics[width=\linewidth]{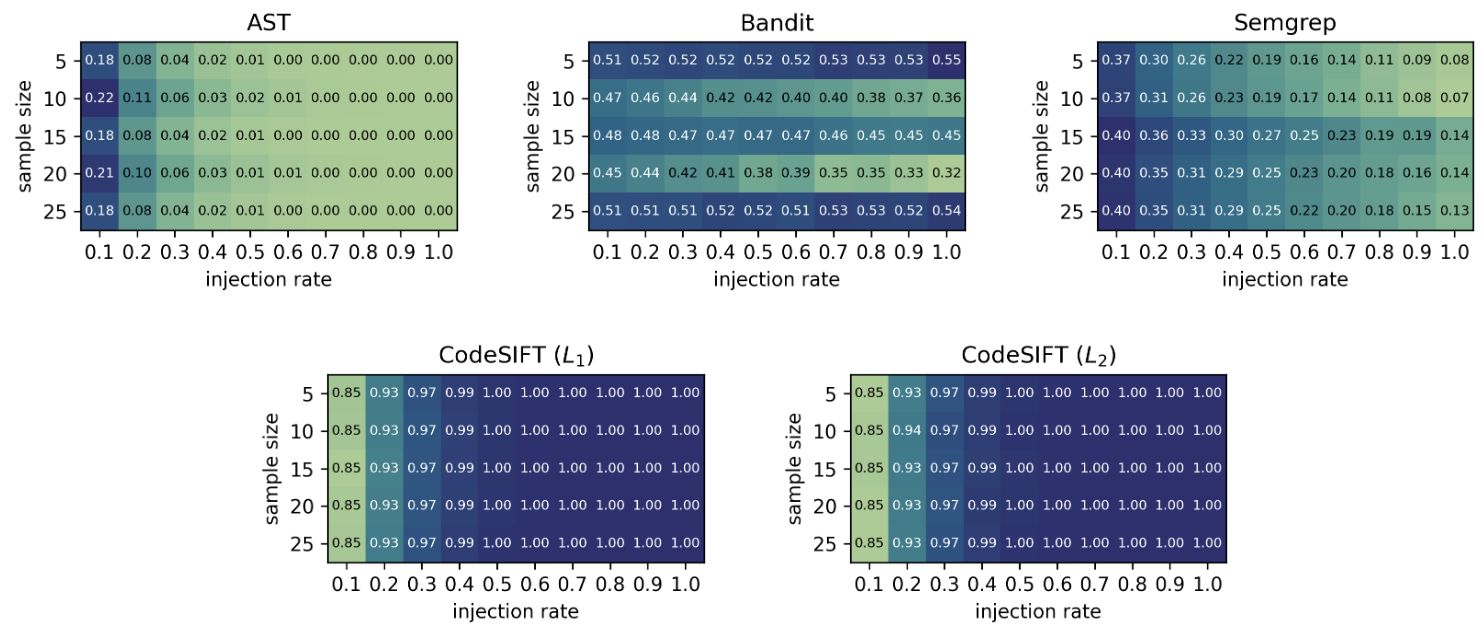}
\caption{AUROC on AuthSec for Granite-3B-Code-Instruct-2K, as a joint function of injection rate and number of sampled completions per prompt, for the three static baselines (top row) and CodeSIFT ($L_1$, $L_2$, bottom row), at a fixed dataset size. CodeSIFT's AUROC depends almost entirely on injection rate and is largely insensitive to the number of sampled completions. The baselines show no clear dependence on either axis.}
\label{fig:auth-granite-3b-code-instruct-2k-injection_rate_n_samples}
\end{figure}

\begin{figure}[ht]
\centering
\includegraphics[width=\linewidth]{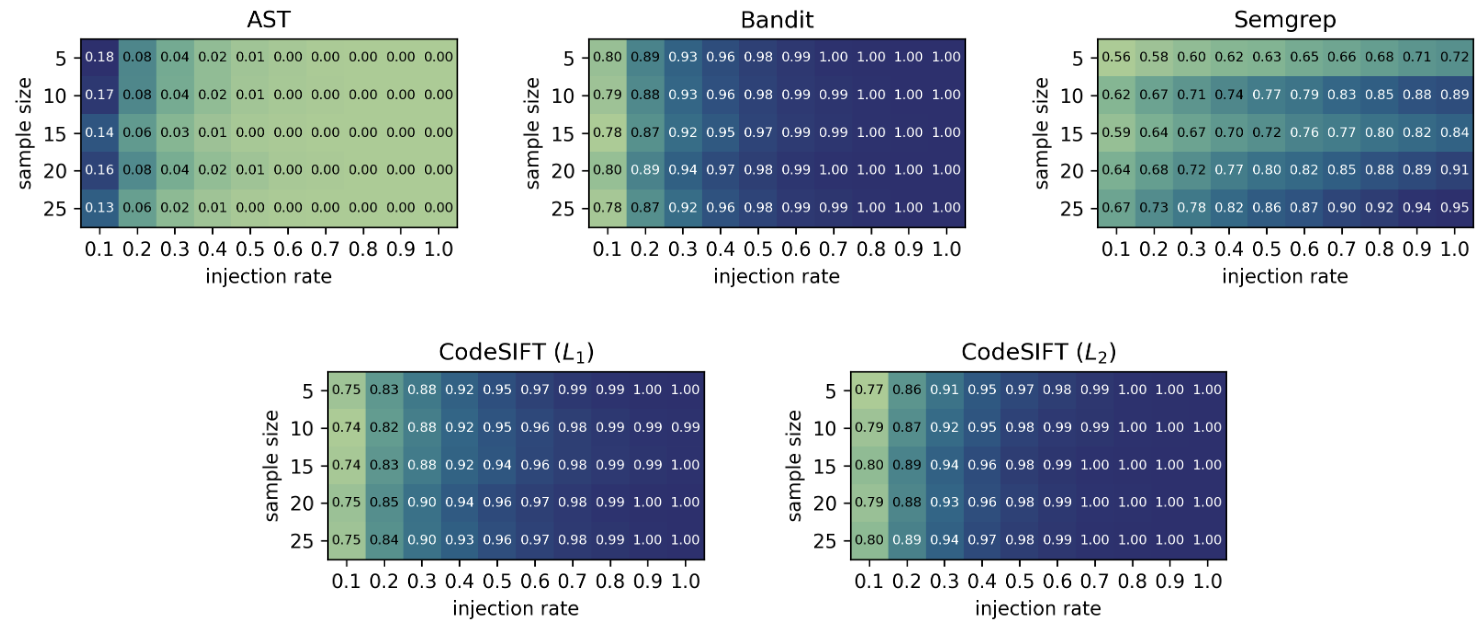}
\caption{AUROC on AuthSec for DeepSeek-Coder-6.7B-Instruct, as a joint function of injection rate and number of sampled completions per prompt, for the three static baselines (top row) and CodeSIFT ($L_1$, $L_2$, bottom row), at a fixed dataset size. CodeSIFT's AUROC depends almost entirely on injection rate and is largely insensitive to the number of sampled completions. The baselines show no clear dependence on either axis.}
\label{fig:auth-deepseek-coder-6_7b-instruct-injection_rate_n_samples}
\end{figure}

\begin{figure}[ht]
\centering
\includegraphics[width=\linewidth]{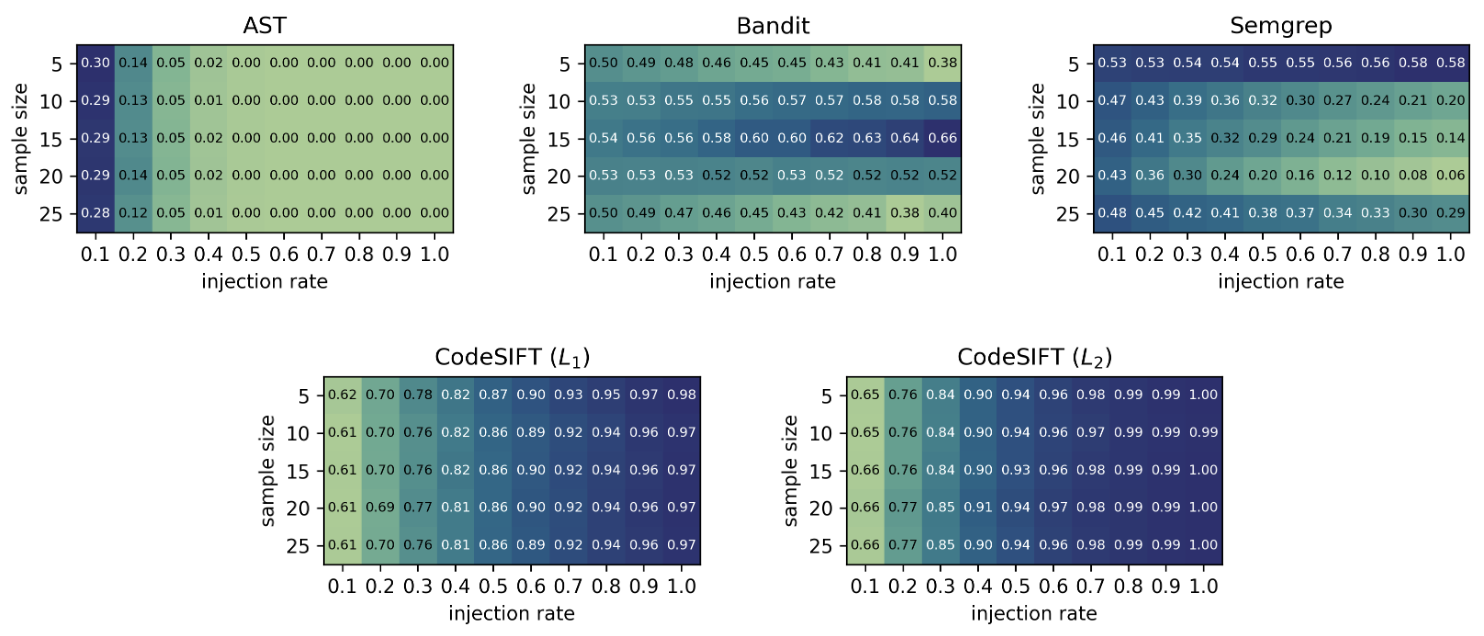}
\caption{AUROC on AuthSec for Qwen2.5-Coder-7B-Instruct, as a joint function of injection rate and number of sampled completions per prompt, for the three static baselines (top row) and CodeSIFT ($L_1$, $L_2$, bottom row), at a fixed dataset size. CodeSIFT's AUROC depends almost entirely on injection rate and is largely insensitive to the number of sampled completions. The baselines show no clear dependence on either axis.}
\label{fig:auth-Qwen2_5-Coder-7B-Instruct-injection_rate_n_samples}
\end{figure}

\begin{figure}[ht]
\centering
\includegraphics[width=\linewidth]{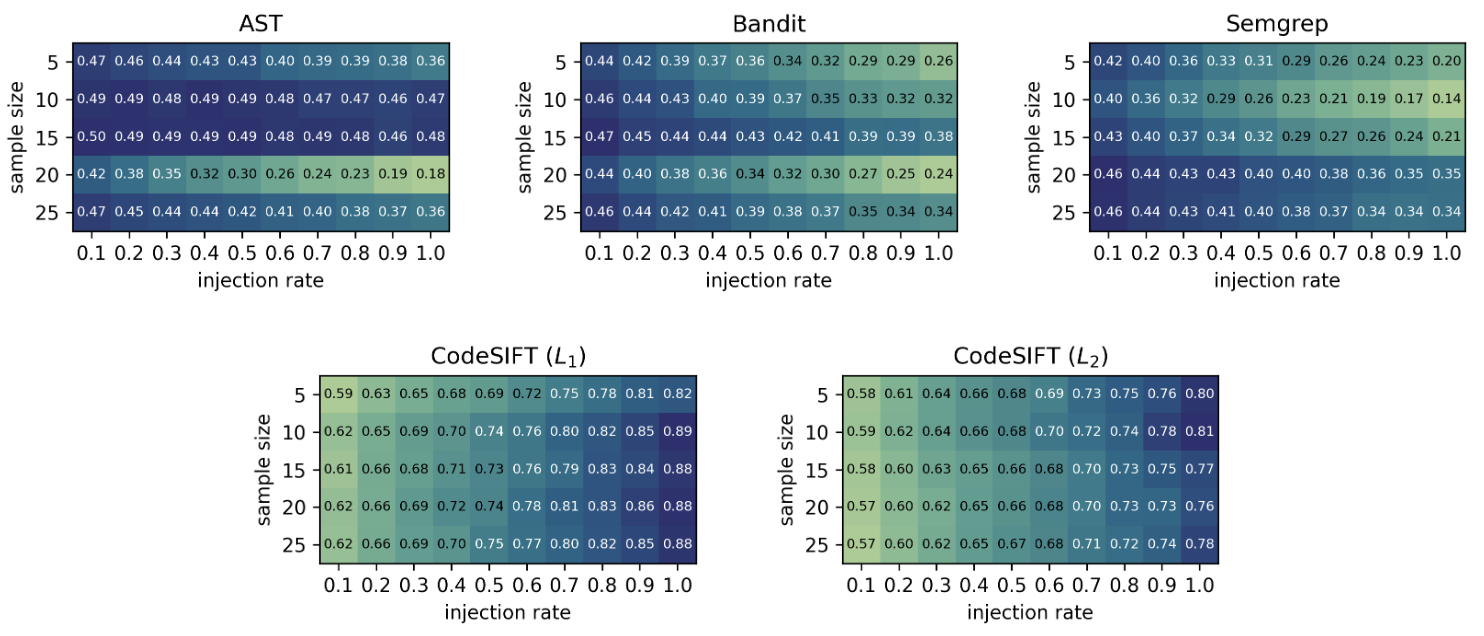}
\caption{AUROC on InfraCloud for Granite-3B-Code-Instruct-2K, as a joint function of injection rate and number of sampled completions per prompt, for the three static baselines (top row) and CodeSIFT ($L_1$, $L_2$, bottom row), at a fixed dataset size. CodeSIFT's AUROC depends almost entirely on injection rate and is largely insensitive to the number of sampled completions. The baselines show no clear dependence on either axis.}
\label{fig:infrastructure-granite-3b-code-instruct-2k-injection_rate_n_samples}
\end{figure}

\begin{figure}[ht]
\centering
\includegraphics[width=\linewidth]{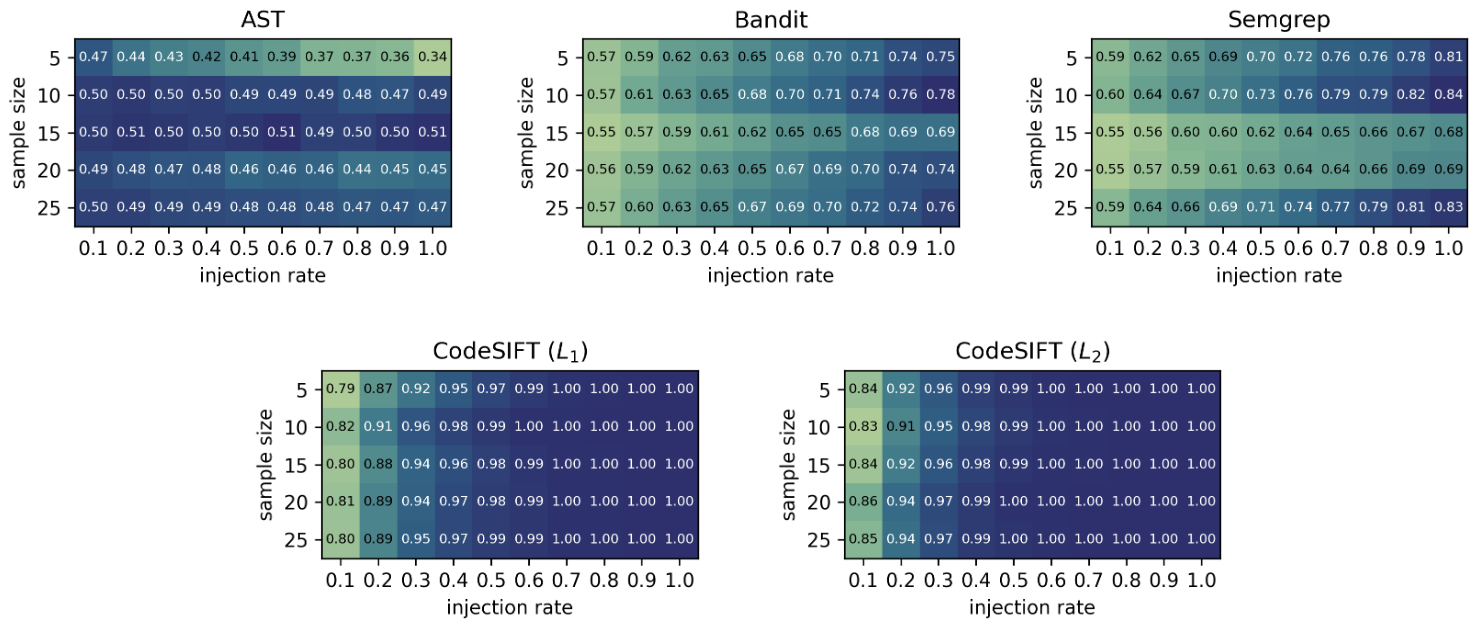}
\caption{AUROC on InfraCloud for DeepSeek-Coder-6.7B-Instruct, as a joint function of injection rate and number of sampled completions per prompt, for the three static baselines (top row) and CodeSIFT ($L_1$, $L_2$, bottom row), at a fixed dataset size. CodeSIFT's AUROC depends almost entirely on injection rate and is largely insensitive to the number of sampled completions. The baselines show no clear dependence on either axis.}
\label{fig:infrastructure-deepseek-coder-6_7b-instruct-injection_rate_n_samples}
\end{figure}

\begin{figure}[ht]
\centering
\includegraphics[width=\linewidth]{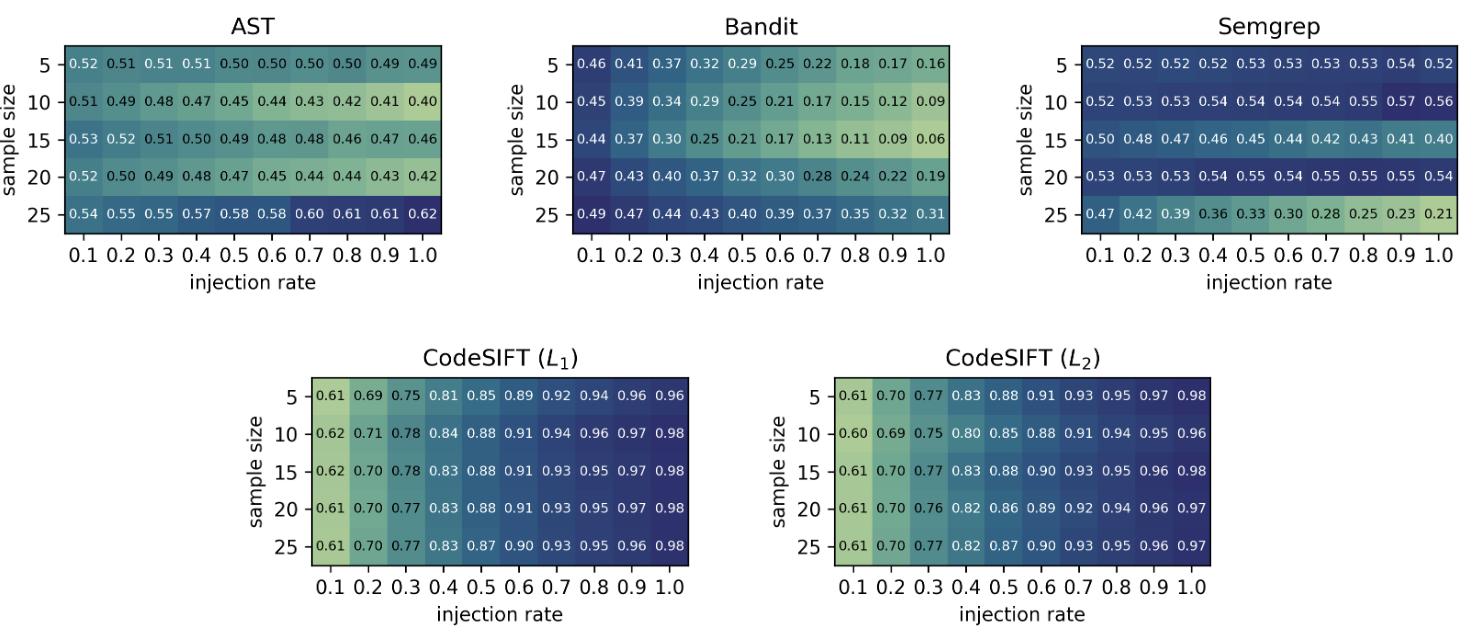}
\caption{AUROC on InfraCloud for Qwen2.5-Coder-7B-Instruct, as a joint function of injection rate and number of sampled completions per prompt, for the three static baselines (top row) and CodeSIFT ($L_1$, $L_2$, bottom row), at a fixed dataset size. CodeSIFT's AUROC depends almost entirely on injection rate and is largely insensitive to the number of sampled completions. The baselines show no clear dependence on either axis.}
\label{fig:infrastructure-Qwen2_5-Coder-7B-Instruct-injection_rate_n_samples}
\end{figure}

\FloatBarrier
\end{document}